\newif\ifEBTEIEEEBuild
\ifdefined\EBTETOPS
  \documentclass[acmsmall,review,screen]{acmart}
  \EBTEIEEEBuildfalse
  \renewcommand\footnotetextcopyrightpermission[1]{}
  \acmJournal{TOPS}
  \acmYear{2026}
  \acmMonth{7}
\else
  \documentclass[conference,onecolumn]{IEEEtran}
  \EBTEIEEEBuildtrue
  \IEEEoverridecommandlockouts
\fi

\ifEBTEIEEEBuild
  \usepackage{cite}
  \usepackage[margin=1in]{geometry}
\fi
\usepackage{amsmath,amssymb,amsfonts,amsthm}
\usepackage{booktabs}
\usepackage{array}
\usepackage{enumitem}
\usepackage{graphicx}
\usepackage{xcolor}
\usepackage{listings}
\usepackage{hyperref}
\usepackage{url}
\usepackage{microtype}
\usepackage{tikz}
\usepackage{placeins}
\ifEBTEIEEEBuild
  \usepackage{fancyhdr}
\fi
\usepackage{longtable}
\ifEBTEIEEEBuild\else
\fi
\usetikzlibrary{arrows.meta,positioning,fit,calc}

\hypersetup{
  hidelinks,
  pdftitle={Explanation-Bound Tool Execution for AI Agents: Server-Verified Action Claims Without Trusting Model Rationales},
  pdfauthor={Genliang Zhu and Chu Wang},
  pdfsubject={A claim-carrying mediation mechanism for governed AI-agent tool execution},
  pdfkeywords={AI agents, tool execution, action claims, auditability, access control, prompt injection, least privilege}
}

\setlist[itemize]{leftmargin=1.2em}
\setlist[enumerate]{leftmargin=1.4em}
\newcommand{\actionclaimmediation}{explanation-bound action-claim mediation}

\newcommand{\facts}{\mathcal{F}}
\newcommand{\decision}{\mathsf{Decision}}

\newcommand{\allow}{\ensuremath{\mathsf{Allow}}}
\newcommand{\review}{\ensuremath{\mathsf{Review}}}
\newcommand{\deny}{\ensuremath{\mathsf{Deny}}}

\newcommand{\mcode}[1]{\text{\texttt{#1}}}
\newcolumntype{L}[1]{>{\raggedright\arraybackslash}p{#1}}

\newcommand{\TaskFamilies}{8}
\newcommand{\VariantsPerFamily}{17}
\newcommand{\ConformanceTasks}{136}
\newcommand{\ConformanceRows}{544}
\newcommand{\HardFixtures}{96}
\newcommand{\SoftFixtures}{32}
\newcommand{\AlignedFixtures}{8}
\newcommand{\MetamorphicChecks}{232}
\newcommand{\StructuralPacketChecks}{136}
\newcommand{\AuthoritativePolicyDenyChecks}{8}
\newcommand{\AuthoritativePolicyMalformedDenyChecks}{8}
\newcommand{\AuthoritativePolicyReviewChecks}{8}
\newcommand{\SchemaSoftHardDominanceChecks}{8}
\newcommand{\RawSchemaHardDominanceChecks}{8}
\newcommand{\HardSoftPacketTruthChecks}{8}
\newcommand{\StaleIncompleteEvidenceChecks}{8}
\newcommand{\MixedMalformedAlignedSetChecks}{8}
\newcommand{\MixedMalformedKnownHardChecks}{8}
\newcommand{\SelfDenialMalformedChecks}{8}
\newcommand{\UnavailableAuthoritativeEvidenceChecks}{8}
\newcommand{\TypedRawDisclosureChecks}{8}
\newcommand{\MalformedExpectedEffectChecks}{8}
\newcommand{\EffectUpperBoundChecks}{4}
\newcommand{\EffectSafeSlackChecks}{4}
\newcommand{\InapplicableEffectClaimChecks}{8}
\newcommand{\ContextExactnessHardChecks}{8}
\newcommand{\UnavailableAuthoritativeContextChecks}{8}
\newcommand{\DuplicateCurrentContextChecks}{3}
\newcommand{\AuthoritativeFactAvailabilityChecks}{8}
\newcommand{\OrthogonalAuthoritativeMalformedHardChecks}{8}
\newcommand{\EmptyAuthoritativeObjectChecks}{8}
\newcommand{\MissingPayloadEnvelopeSiblingChecks}{8}
\newcommand{\MissingIntentEnvelopeSiblingChecks}{5}
\newcommand{\PresentOptionalOrthogonalHardChecks}{8}
\newcommand{\ResourceOperandAvailabilityChecks}{4}
\newcommand{\InvalidAuthoritativeNumericDomainChecks}{8}
\newcommand{\SelfDenialUnavailableAuthorityChecks}{8}
\newcommand{\PolicyBasisOperandAvailabilityChecks}{8}
\newcommand{\CanonicalScopeClaimChecks}{3}
\newcommand{\EvidenceExtraKeyStaleChecks}{8}
\newcommand{\WideAllowedFieldReplayChecks}{1}
\newcommand{\OverCapAuthoritativeSetChecks}{8}
\newcommand{\TailBoundedClaimChecks}{8}
\newcommand{\TailAllowedFieldClaimChecks}{1}
\newcommand{\TailContextDependencyChecks}{3}
\newcommand{\CanonicalInvalidScalarChecks}{8}
\newcommand{\InvalidAuthoritativeContextSourceChecks}{3}
\newcommand{\ClaimAliasIsolationChecks}{4}
\newcommand{\AuthoritativeEnumAvailabilityChecks}{4}
\newcommand{\PacketDiagnosticCollisionChecks}{2}
\newcommand{\PacketOutcomeOrderChecks}{2}
\newcommand{\ScopeFreePolicyBasisChecks}{8}
\newcommand{\EmptyScopeMissingChecks}{8}
\newcommand{\UnicodeLengthChecks}{2}
\newcommand{\BoundedValidationResourceChecks}{2}
\newcommand{\EffectClaimOmissionChecks}{4}
\newcommand{\DestinationClaimOmissionChecks}{1}
\newcommand{\FieldClaimOmissionChecks}{1}
\newcommand{\JoinCompositionChecks}{27}
\newcommand{\RuntimeFamilies}{4}
\newcommand{\RuntimeTasks}{68}
\newcommand{\RuntimeRows}{272}
\newcommand{\RuntimeHardFixtures}{48}
\newcommand{\RuntimeSoftFixtures}{16}
\newcommand{\RuntimeAlignedFixtures}{4}
\newcommand{\HostedModels}{4}
\newcommand{\HostedTasks}{12}
\newcommand{\HostedAlignedTasks}{4}
\newcommand{\HostedRepairableTasks}{4}
\newcommand{\HostedNonrepairableTasks}{4}
\newcommand{\HostedRepeats}{2}
\newcommand{\HostedInitialPerModel}{24}
\newcommand{\HostedRepairPerModel}{24}
\newcommand{\HostedSeededPerModel}{8}
\newcommand{\HostedRows}{224}
\newcommand{\HostedInitialRows}{96}
\newcommand{\HostedRepairRows}{96}
\newcommand{\HostedSeededRows}{32}
\newcommand{\HostedProviderErrors}{21}
\newcommand{\HostedInitialAgreementAll}{71/96}
\newcommand{\HostedRepairAgreementAll}{66/96}
\newcommand{\HostedSeededAgreementAll}{19/32}
\newcommand{\HostedInitialAgreementResponse}{71/88}
\newcommand{\HostedRepairAgreementResponse}{66/84}
\newcommand{\HostedSeededAgreementResponse}{19/31}
\newcommand{\HostedCorrectToCorrect}{63}
\newcommand{\HostedCorrectToIncorrect}{8}
\newcommand{\HostedIncorrectToCorrect}{3}
\newcommand{\HostedIncorrectToIncorrect}{22}
\newcommand{\HostedNonrepairableRows}{32}
\newcommand{\HostedNonrepairableDeny}{17}
\newcommand{\HostedNonrepairableReview}{15}
\newcommand{\HostedExpectedDenyRows}{64}
\newcommand{\HostedCurrentAttempts}{224}
\newcommand{\HostedCurrentResponses}{203}
\newcommand{\HostedCurrentProviderFailures}{21}
\newcommand{\HostedCurrentSavedClaims}{160}
\newcommand{\HostedCurrentParseFailures}{43}
\newcommand{\HostedHistoricalConformantRows}{149}
\newcommand{\HostedCurrentInitialAgreementAll}{70/96}
\newcommand{\HostedCurrentRepairAgreementAll}{65/96}
\newcommand{\HostedCurrentSeededAgreementAll}{17/32}
\newcommand{\HostedCurrentInitialAgreementResponse}{70/88}
\newcommand{\HostedCurrentRepairAgreementResponse}{65/84}
\newcommand{\HostedCurrentSeededAgreementResponse}{17/31}
\newcommand{\HostedCurrentAllows}{56}
\newcommand{\HostedCurrentReviews}{75}
\newcommand{\HostedCurrentDenies}{93}
\newcommand{\HostedCurrentDecisionChanges}{56}
\newcommand{\HostedCurrentAllowToReview}{5}
\newcommand{\HostedCurrentAllowToDeny}{4}
\newcommand{\HostedCurrentReviewToDeny}{47}
\newcommand{\HostedCurrentDenyDowngrades}{0}
\newcommand{\HostedCurrentNonrepairableDeny}{26}
\newcommand{\HostedCurrentNonrepairableReview}{6}
\newcommand{\HostedCurrentNonrepairableAllow}{0}
\newcommand{\HostedCurrentExpectedDenyAllow}{0}
\newcommand{\HostedCurrentProviderCalls}{0}
\newcommand{\HostedCurrentFingerprintDisplay}{\texttt{sha256:8fdb3c704c964a2c\allowbreak 9474c190b4cdafca\allowbreak 8c08890bce6f4f37\allowbreak 36ca1aabd109abb1}}
\newcommand{\AlignedAllows}{5}
\newcommand{\AlignedReviews}{3}
\newcommand{\ExternalProposals}{23}
\newcommand{\ExternalRows}{92}
\newcommand{\ExternalAttacks}{12}
\newcommand{\ExternalBenign}{8}
\newcommand{\ExternalDirect}{3}
\newcommand{\ExternalConfigurations}{4}
\newcommand{\ExternalOrdinaryTasks}{4}

\newcommand{\FreePDA}{0.1471}
\newcommand{\FreeSNAC}{0.3750}
\newcommand{\FreeHFAR}{0.6250}
\newcommand{\FreeSFAR}{0.6250}
\newcommand{\FreeAFA}{1.0000}
\newcommand{\FreeRRF}{0.3750}
\newcommand{\FreeSEC}{0.0000}
\newcommand{\FreeEFLA}{1.0000}
\newcommand{\SchemaPDA}{0.3162}
\newcommand{\SchemaSNAC}{0.5313}
\newcommand{\SchemaHFAR}{0.5729}
\newcommand{\SchemaSFAR}{0.1563}
\newcommand{\SchemaAFA}{1.0000}
\newcommand{\SchemaRRF}{0.4632}
\newcommand{\SchemaSEC}{0.0000}
\newcommand{\SchemaEFLA}{0.0000}
\newcommand{\PayloadPDA}{0.5515}
\newcommand{\PayloadSNAC}{0.6875}
\newcommand{\PayloadHFAR}{0.3646}
\newcommand{\PayloadSFAR}{0.1563}
\newcommand{\PayloadAFA}{1.0000}
\newcommand{\PayloadRRF}{0.3750}
\newcommand{\PayloadSEC}{0.0000}
\newcommand{\PayloadEFLA}{0.0000}
\newcommand{\FullPDA}{1.0000}
\newcommand{\FullSNAC}{1.0000}
\newcommand{\FullHFAR}{0.0000}
\newcommand{\FullSFAR}{0.0000}
\newcommand{\FullAFA}{1.0000}
\newcommand{\FullRRF}{0.2574}
\newcommand{\FullSEC}{0.6471}
\newcommand{\FullEFLA}{0.0000}
\newcommand{\RuntimeFreeHCFD}{1.0000}
\newcommand{\RuntimeFreeSoftDraft}{1.0000}
\newcommand{\RuntimeFreeAlignedDraft}{1.0000}
\newcommand{\RuntimeFreeAudit}{0.9412}
\newcommand{\RuntimeFreeSEC}{0.0000}
\newcommand{\RuntimeFreeExec}{0.0000}
\newcommand{\RuntimeFreeVerifyP}{0.0027}
\newcommand{\RuntimeFreeTotalP}{2.5756}
\newcommand{\RuntimeSchemaHCFD}{0.9167}
\newcommand{\RuntimeSchemaSoftDraft}{1.0000}
\newcommand{\RuntimeSchemaAlignedDraft}{1.0000}
\newcommand{\RuntimeSchemaAudit}{0.9375}
\newcommand{\RuntimeSchemaSEC}{0.0000}
\newcommand{\RuntimeSchemaExec}{0.0000}
\newcommand{\RuntimeSchemaVerifyP}{0.1231}
\newcommand{\RuntimeSchemaTotalP}{1.8426}
\newcommand{\RuntimePayloadHCFD}{0.5833}
\newcommand{\RuntimePayloadSoftDraft}{1.0000}
\newcommand{\RuntimePayloadAlignedDraft}{1.0000}
\newcommand{\RuntimePayloadAudit}{0.9167}
\newcommand{\RuntimePayloadSEC}{0.0000}
\newcommand{\RuntimePayloadExec}{0.0000}
\newcommand{\RuntimePayloadVerifyP}{0.0943}
\newcommand{\RuntimePayloadTotalP}{1.4660}
\newcommand{\RuntimeFullHCFD}{0.0000}
\newcommand{\RuntimeFullSoftDraft}{1.0000}
\newcommand{\RuntimeFullAlignedDraft}{1.0000}
\newcommand{\RuntimeFullAudit}{1.0000}
\newcommand{\RuntimeFullSEC}{0.6471}
\newcommand{\RuntimeFullExec}{0.0000}
\newcommand{\RuntimeFullVerifyP}{0.3289}
\newcommand{\RuntimeFullTotalP}{1.4453}
\newcommand{\RuntimeSourceFiles}{86}

\newcommand{\RuntimeSourceDigestDisplay}{\texttt{sha256:9ce8bf05053bfe74\allowbreak 4b8eca2d6ae704d0\allowbreak dcb664ec17cfa277\allowbreak dae6ce5380aad2fc}}
\newcommand{\AgentDojoRevisionDisplay}{\texttt{a75aba7631d3ca5f\allowbreak b7ab938965c97ead\allowbreak 2f9ff84b}}
\newcommand{\ModelGPTTableRow}{GPT-OSS 120B & 24/24 & 24/24 & 24/24 & 23/24 & 8/8 & 0/56 & 0/8 & 0/56 \\}
\newcommand{\ModelQwenTableRow}{Qwen3 Next 80B & 24/24 & 24/24 & 24/24 & 24/24 & 2/8 & 0/56 & 0/8 & 0/56 \\}
\newcommand{\ModelNemotronTableRow}{Nemotron 3 Nano 30B & 7/24 & 2/24 & 7/24 & 7/24 & 2/8 & 0/56 & 0/8 & 0/56 \\}
\newcommand{\ModelDeepSeekTableRow}{DeepSeek V4 Flash & 16/24 & 16/24 & 16/24 (16/16) & 12/24 (12/12) & 7/8 (7/7) & 21/56 & 0/8 & 0/56 \\}
\newcommand{\AgentAblationTableRow}[1]{#1 & 0.4783 & 0.0000 & 1.0000 & 1.0000 & 1.0000 & 0.0000 \\}
\newcommand{\AgentFullTableRow}{Full profile & 1.0000 & 1.0000 & 1.0000 & 1.0000 & 1.0000 & 1.0000 \\}

\definecolor{opBlue}{HTML}{3B82F6}
\definecolor{opBlueFill}{HTML}{EFF6FF}
\definecolor{opBlueText}{HTML}{1E40AF}
\definecolor{opGreen}{HTML}{22C55E}
\definecolor{opGreenFill}{HTML}{F0FDF4}
\definecolor{opGreenText}{HTML}{15803D}
\definecolor{opAmber}{HTML}{FCD34D}
\definecolor{opAmberFill}{HTML}{FEF3C7}
\definecolor{opAmberText}{HTML}{92400E}
\definecolor{opRed}{HTML}{EF4444}
\definecolor{opRedFill}{HTML}{FEF2F2}
\definecolor{opRedText}{HTML}{DC2626}
\definecolor{opPurple}{HTML}{8B5CF6}
\definecolor{opPurpleFill}{HTML}{F5F3FF}
\definecolor{opPurpleText}{HTML}{7C3AED}

\newtheorem{proposition}{Proposition}

\newtheorem{definition}{Definition}

\lstdefinelanguage{json}{
  morestring=[b]",
  showstringspaces=false,
  morecomment=[l]{//},
  morecomment=[s]{/*}{*/},
  keywords={true,false,null},
  sensitive=false
}

\ifEBTEIEEEBuild
  \title{Explanation-Bound Tool Execution for AI Agents:\\Server-Verified Action Claims Without Trusting Model Rationales}
\else
  \title{Explanation-Bound Tool Execution for AI Agents: Server-Verified Action Claims Without Trusting Model Rationales}
\fi
\ifEBTEIEEEBuild
  \author{%
  \parbox[t]{0.48\linewidth}{\centering
    Genliang Zhu\\
    Accentrust\\
    Georgia Institute of Technology
  }
  \hfill
  \parbox[t]{0.48\linewidth}{\centering
    Chu Wang\\
    Accentrust\\
    University of Illinois Urbana-Champaign
  }
  \thanks{Research stewardship: Accentrust. Correspondence: \href{mailto:research@accentrust.com}{research@accentrust.com}.}%
  }
  \date{}
\else
  \author{Genliang Zhu}
  \affiliation{\institution{Accentrust}\country{Canada}}
  \affiliation{
    \institution{Georgia Institute of Technology}
    \city{Atlanta}
    \state{Georgia}
    \country{United States}
  }
  \email{research@accentrust.com}
  \author{Chu Wang}
  \affiliation{\institution{Accentrust}\country{Canada}}
  \affiliation{
    \institution{University of Illinois Urbana-Champaign}
    \city{Urbana-Champaign}
    \state{Illinois}
    \country{United States}
  }
  
\fi

\begin{document}
\ifEBTEIEEEBuild
  \maketitle
  \pagestyle{fancy}
  \fancyhf{}
  \renewcommand{\headrulewidth}{0pt}
  \renewcommand{\footrulewidth}{0pt}
  \fancyfoot[R]{\thepage}
  \thispagestyle{fancy}
\fi

\begin{abstract}
Tool-using agents expose structured calls but commonly attach free-form rationales. Such rationales are neither authorization nor reliable introspection. We present an explanation-bound action-claim mediation layer that converts decision-relevant rationale content into typed action claims and checks them against server-held intent, policy, payload, tool, risk, provenance, and freshness facts. The layer cannot widen baseline authority: conflicts deny, incomplete or uncertain claims review, and only matching claims remain eligible for governed execution. We formalize this composition under explicit mediation and trusted-fact assumptions and implement a versioned reference profile with minimized audit packets. Across \ConformanceTasks{} authored conformance scenarios, the full profile matches all specified dispositions, admits none of \HardFixtures{} designated hard contradictions, and passes \MetamorphicChecks{} metamorphic checks. A draft-only reference integration forwards none of \RuntimeHardFixtures{} authored hard cases under the full profile while preserving all \RuntimeSoftFixtures{} soft-review and \RuntimeAlignedFixtures{} aligned draft paths. In a frozen 2026-07-12 exploratory \HostedRows-attempt hosted-model record, the historical generation/runner agreement counts are \HostedInitialAgreementAll{}, \HostedRepairAgreementAll{}, and \HostedSeededAgreementAll{}; a zero-call revalidation of the preserved minimized claims under the current pipeline yields \HostedCurrentInitialAgreementAll{}, \HostedCurrentRepairAgreementAll{}, and \HostedCurrentSeededAgreementAll{}. In an AgentDojo-derived semantic check, existing high-risk controls make all \ExternalAttacks{} attack proposals non-allow, while the full-profile check resolves the task--proposal contradictions as deny. Together, these studies establish profile conformance and demonstrate the feasibility of server-checked action claims within the evaluated settings.
\end{abstract}

\ifEBTEIEEEBuild
\else
  \maketitle
\fi

\section{Introduction}

AI agents increasingly act through application tools. They query private data, create records, export collections, update shared state, and invoke administrative operations. The action interface is usually structured: a tool name and payload are validated against a schema. The explanation interface is not. Some systems retain free-form model reasoning, some attach a short justification string, and others provide no decision basis beyond the selected tool itself.

This asymmetry creates a governance gap. A payload can be syntactically valid and statically authorized while remaining inconsistent with the user's request. A justification can sound reasonable while referring to a different resource, omitting an external destination, understating risk, or hiding that an untrusted document influenced the decision. Language-model explanations can also be unfaithful to the factors that caused a prediction~\cite{turpin2023unfaithful,lanham2023faithfulness,yee2024dissociation}. Consequently, preserving a rationale is not equivalent to preserving trustworthy evidence.

The problem is not solved by requiring chain-of-thought. ReAct demonstrates the utility of interleaving reasoning and actions~\cite{yao2023react}, but operational governance needs a smaller and more stable object than an unrestricted reasoning trace. Raw traces can contain private inputs, untrusted instructions, proprietary prompts, or irrelevant internal deliberation. They are also difficult to validate. A security gateway should instead ask a bounded question: \emph{what decision-relevant claims must be true for this proposed tool effect to be justified, and can the server verify them independently?}

We propose \actionclaimmediation{}, which requires a structured explanation object, \mbox{\texttt{ToolInvocationExplanation}}. The object contains a declared intent class, selected tool, policy basis, expected effect, risk tier, uncertainty, untrusted-context dependencies, and evidence references. The mediation layer treats this object as an untrusted claim set and derives no authority from it. A server-side verifier compares enforceable claims with independently held facts: a current intent certificate, canonical tool metadata, the proposed payload, a policy snapshot, an authorized route record, and a context-risk snapshot. The reference profile carries digest references for intent, policy, and payload. Route eligibility remains a server fact, and each context dependency carries a source-category digest.

The verifier produces one of three outcomes. \emph{Allow} means the explanation is consistent with independently authorized low- or medium-risk facts. \emph{Review} means the proposal is potentially legitimate but contains uncertainty, incomplete policy evidence, risk-sensitive effects, or another condition requiring operator judgment. \emph{Deny} means the explanation contradicts a hard execution fact, omits a material untrusted dependency, or attempts to place raw sensitive context in durable evidence. These outcomes compose with, rather than replace, existing authorization and effect controls.

Our contributions are:

\begin{itemize}
  \item \textbf{Action-claim abstraction.} We separate model rationale from a typed, claim-carrying mediation object whose enforceable fields are checked against authoritative application state and whose advisory fields can only tighten routing.
  \item \textbf{Explicit-assumption composition.} We define a three-outcome decision lattice in which baseline authorization, the action-claim disposition, and effect controls combine by their most restrictive outcome; review is non-executing, and any later approval must resolve current facts and recompute the composed decision.
  \item \textbf{Executable profile and evidence packet.} We provide a bounded schema, stable reason taxonomy, strict reference verifier, minimized structural audit projection, conformance profile, and cross-artifact integrity verifier.
  \item \textbf{Tiered mechanism evidence.} We evaluate authored profile conformance and predicate ablation, draft-only runtime integration, hosted-model generation, and transfer to an external benchmark structure across \ConformanceRows{}, \RuntimeRows{}, \HostedRows{}, and \ExternalRows{} rows, respectively.
\end{itemize}

\section{Background and Motivation}

\subsection{Explanation Is Not Authorization}

Explanations in machine learning are commonly used to help people interpret model outputs~\cite{ribeiro2016lime,miller2019explanation}. Interpretability is not a single property, and practical users can misunderstand or overgeneralize what tools reveal~\cite{lipton2018mythos,kaur2020interpreting,poursabzi2018manipulating}. In an agent gateway, the immediate question is not whether a human finds an explanation persuasive. It is whether the proposed external effect remains inside independently authorized bounds. Classical least privilege requires minimizing authority~\cite{saltzer1975}; attribute-based access control evaluates trusted attributes and policy rather than prose~\cite{nist2014abac}. The explanation-bound execution check applies the same discipline to explanations: prose may summarize a decision, but enforceable fields must be checked against authoritative state.

This distinction is especially important because current models can produce plausible post hoc rationales. Work on chain-of-thought faithfulness shows that stated reasoning may omit influential biasing features, vary under interventions, or diverge from the process that produced an answer~\cite{turpin2023unfaithful,lanham2023faithfulness,yee2024dissociation}. The mediation layer instead verifies an externally testable contract about the action being requested.

\subsection{Tool-Using Agent Threats}

Indirect prompt injection and sandbox studies show that untrusted content can redirect integrated applications and tool-using agents~\cite{greshake2023indirect,ruan2024toolemu}. AgentDojo and InjecAgent systematize harmful and exfiltrating action tasks~\cite{debenedetti2024agentdojo,zhan2024injecagent}. ToolHijacker and MCPTox demonstrate that tool descriptions and catalogs are themselves attack surfaces~\cite{shi2025toolhijacker,wang2025mcptox}. StruQ separates instructions from data at the model input boundary~\cite{chen2024struq}; AttriGuard studies causal attribution for tool invocations~\cite{he2026attriguard}. The action-claim mediation layer is complementary: it assumes that model-level defenses and attribution can fail, then checks whether the explanation's effect claims agree with server-held facts.

Repository-level evidence also shows that agent applications can contain taint-style paths from adversarial prompts to security-sensitive operations: AgentFuzz reported such vulnerabilities in widely used open-source agent applications~\cite{liu2025agentfuzz}. Vulnerability discovery and action-time mediation address different failure modes. The former locates exploitable implementation paths; the latter constrains a mediated invocation using current application facts. The proposed check does not detect arbitrary code vulnerabilities and provides no guarantee for effects that bypass its gateway.

MCP identifies tools as model-controlled primitives and specifies transport authorization semantics~\cite{mcp2025tools,mcp2025authorization}. Token validation is necessary, but a valid token does not establish that a particular action is supported by the current user request. NIST's Generative AI Profile emphasizes documentation, testing, human oversight, and monitoring~\cite{nist2024genai}; OWASP recommends least privilege and approval for high-impact actions~\cite{owasp2025promptinjection}. The proposed check operationalizes a narrow part of that guidance at the tool-effect boundary.

\subsection{Auditability and Minimal Disclosure}

Recent work distinguishes accountability, auditability, and the act of auditing agent behavior~\cite{nian2026auditable}. Static analysis can identify risky agent code and configuration~\cite{zhang2026agentaudit}, but post-deployment reconstruction also requires runtime decision evidence. A complete raw transcript is neither necessary nor always safe. Evidence digests should use a deterministic serialization such as JCS when JSON objects are hashed~\cite{rundgren2020jcs}. The reference profile records normalized categories, reason codes, bounded identifiers, and cryptographic digests. The structural evaluation covers packet completeness and rejection of an explicit raw-content field and fixed synthetic marker; semantic content controls and operator comprehension are outside its scope.

\section{Problem Definition and Threat Model}

\subsection{System Model}

An agent receives user request $u$, observes context $X$, selects tool $t$, constructs payload $x$, and emits explanation $e$. The application gateway holds independently derived facts:

\[
\facts = (C, T, P, R, X_s, Q, H),
\]

where $C$ is an intent certificate, $T$ is trusted canonical tool metadata, $P$ is the current application policy, $R$ is a route record or authorized candidate set, $X_s$ is a current context-risk snapshot, $Q$ is preflight or state-witness evidence, and $H$ is integrity-protected audit state. Model-visible tool descriptions may be attacker-controlled; $T$ denotes a separately governed registry entry. Not every deployment must implement every component, but a missing fact cannot be silently replaced by a model assertion. The public profile evaluated here directly resolves $C,T,P,R,$ and $X_s$; execution-time state witness $Q$ and durable audit integrity $H$ are composition requirements rather than claims established by the draft-only prototype.

The gateway first evaluates ordinary authorization and effect controls. The action-claim mediation layer then verifies explanation claims against $\facts$ and selects an execution mode. For high-risk actions, a consistent explanation can still require preflight, confirmation, or draft review.

\subsection{Adversary Capabilities}

We consider an adversary that can:

\begin{itemize}
  \item influence model-visible documents, messages, tool outputs, or tool descriptions;
  \item cause the planner to select an incorrect tool or construct an overbroad payload;
  \item induce a plausible but incomplete, contradictory, or risk-understated explanation;
  \item omit or mislabel untrusted context dependencies;
  \item replay stale explanation evidence after intent, policy, route, or context changes; and
  \item attempt to place private or adversarial raw content into durable audit records, flood the review queue, or use detailed repair feedback as an adaptive oracle.
\end{itemize}

We do not assume the model, planner, retriever, explanation generator, model-visible tool catalog, or free-text summaries are trusted. We assume that principal and tenant binding, gateway code, the canonical policy and tool stores, digest implementation, clocks or version sources, and scoped current-state resolution are correct. Evidence lookup must be bound to the authenticated application, actor, and session. A compromised fact source, confused-deputy identity binding, malicious code inside an otherwise authorized tool, out-of-band side effects, colluding administrators, and a fully compromised enforcement process remain outside the guarantee envelope.

\paragraph{Lifecycle boundary.}
The security-relevant unit includes both initial verification and any later approval. A \review{} outcome may create a draft but cannot execute an effect. Before acting on an approved draft, the gateway must reload current intent, route, policy, context, payload, and preflight facts and recompute the composed decision. The present runtime study covers only draft creation; the later approval transition remains unevaluated. Rate limiting, coarse model-facing errors, and independently governed audit access contain review-flooding and feedback-oracle risks.

\subsection{Security and Governance Goals}

The explanation-bound execution check targets five goals:

\begin{enumerate}
  \item \textbf{Non-authorization:} explanation content cannot grant authority absent from the underlying policy and effect-control path.
  \item \textbf{Contradiction detection:} material mismatches between explanation and trusted facts cannot yield immediate allow.
  \item \textbf{Review routing:} uncertainty and incomplete soft evidence escalate without being mislabeled as hard attacks.
  \item \textbf{Structural replay support:} a packet identifies the profile, normalized predicate outcomes, fact digests, and reason categories needed for rule-level replay when the referenced facts remain resolvable.
  \item \textbf{Minimal ordinary disclosure:} the default packet excludes raw untrusted content; semantic secret detection and authorized source retrieval remain separate controls.
\end{enumerate}

\section{Explanation Contract}

\subsection{Contract Fields}

Table~\ref{tab:fields} defines the core explanation fields. Human-readable summaries are permitted, but enforcement depends on typed fields and evidence references.

\begin{table}[h]
\centering
\small
\begin{tabular}{L{0.22\linewidth}L{0.31\linewidth}L{0.37\linewidth}}
\toprule
\textbf{Field} & \textbf{Purpose} & \textbf{Independent comparison} \\
\midrule
Intent classes and summary & Declares the request-level purpose. & Intent certificate classes, allowed operation, and request digest; summary text is advisory. \\
Selected tool and reason code & Identifies the proposed capability. & Trusted registry and authorized route record; reason-code prose is advisory. \\
Policy basis & Lists required scopes and rule identifiers. & Current authorization and policy snapshot. \\
Expected effect & Declares operation, resources, bounds, fields, and destination. & Tool semantics, payload, preflight impact, and intent bounds. \\
Risk tier & Declares impact sensitivity. & Server-owned tool and policy risk. \\
Uncertainty & Declares ambiguity or missing evidence. & One-way escalation only; a low declaration cannot establish certainty. \\
Context dependencies & Declares categorical provenance and digests. & Current context-risk source and digest; claimed influence is advisory. \\
Evidence references & Binds intent, policy, and payload evidence. & Exactly three current digest-bearing records in profile v0.1. \\
\bottomrule
\end{tabular}
\caption{Core \texttt{ToolInvocationExplanation} fields. No field is self-authorizing.}
\label{tab:fields}
\end{table}
\FloatBarrier

\begin{table}[h]
\centering
\scriptsize
\resizebox{\linewidth}{!}{%
\begin{tabular}{L{0.18\linewidth}L{0.24\linewidth}L{0.34\linewidth}L{0.20\linewidth}}
\toprule
\textbf{Claim field} & \textbf{Authoritative source} & \textbf{Implemented predicate} & \textbf{Included mutation or status} \\
\midrule
Schema and version & Public JSON Schema and profile & Strict validation before semantic comparison; unknown or malformed input cannot allow. Parseable subfields retain independently known contradictions. & Invalid enumeration, mixed-malformed, and malformed-object groups. \\
Intent and operation & Current intent certificate & Declared classes are confined to certified classes; payload operation and effect are checked against the certified operation set. Unavailable authoritative subfields are unknown, not matches. & Class, operation, payload-outside-intent, and availability controls. \\
Tool and route & Canonical registry and current authorized route & Declared and proposed tool agree, and the tool belongs to the server-owned candidate set. & Tool and authorized-route mismatch. \\
Resource and bounds & Intent certificate, payload, and tool semantics & Resource type and target set agree with the payload; each applicable effect claim covers the payload and remains contained by intent. Substantive claims on inapplicable dimensions deny. & Resource, overbreadth, field, amount, record, destination, upper-bound, safe-slack, and inapplicability controls. \\
Policy basis & Current policy snapshot & Required scopes, rule identifiers, and policy decision match current server facts. & Missing or fabricated rule, scope, and decision mutations. \\
Risk & Canonical tool and policy risk & Understatement escalates to review; declared risk never lowers server risk. & Risk-understatement mutation. \\
Uncertainty & Model declaration and server review policy & High uncertainty can only tighten to review; low uncertainty is not independently certified. & One-way advisory field. \\
Context & Current context source/category and digest & The normalized dependency-pair set is empty for trusted context and exactly the current pair for untrusted context; an unavailable authoritative state is unknown. \texttt{instructionInfluence} is advisory. & Omission, extra-pair, duplicate-pair, and current-state availability controls. \\
Evidence references & Current intent, policy, and payload records & Exactly one current digest for each profile-v0.1 type. Any independently known mismatch denies even when another type is missing or duplicated; unavailable current state is unknown. & Stale, missing, duplicate, compound, and authoritative-availability controls. \\
Summary and tool reason & None; operator-facing only & Bounded and non-authorizing; never used to widen a decision. & Schema and disclosure checks only. \\
\bottomrule
\end{tabular}}
\caption{Field-to-authority contract for the reference profile.}
\label{tab:fieldauthority}
\end{table}
\FloatBarrier

\subsection{Chain-of-Thought Boundary}

The contract is not chain-of-thought. The mediation layer neither requests nor requires hidden
reasoning traces. The \texttt{intentSummary}, \texttt{toolReasonCode},
\texttt{instructionInfluence}, and low-uncertainty declarations are
non-authorizing. Profile v0.1 routes high uncertainty to review and records the
other advisory values without using them to relax or tighten its disposition;
a separately versioned deployment policy may use them only for stricter
handling. Enforceable comparisons use server-checkable values such as
operation, route membership, resource identifier, record bound, risk tier,
destination, policy rule, and evidence digest. This boundary reduces privacy
exposure and avoids treating fluent prose or self-reported cognition as a
security proof.

\subsection{Minimal-Disclosure Evidence}

Context dependencies contain a source category and digest plus an advisory influence flag. The ordinary evidence projection omits raw document text, messages, customer records, proprietary prompts, and model traces. The current artifact rejects an explicit \texttt{rawContent} field and scans for a fixed synthetic marker; bounded free-text summaries remain a possible encoded-disclosure channel. Production use therefore requires semantic content controls, retention policy, and separately authorized source access beyond the structural checks evaluated here.

\section{Formal Model}

\begin{definition}[Explanation projection]
Let $\phi_e(e)$ map explanation $e$ to normalized claims over dimensions
\[
D = \{\text{schema},\text{intent},\text{tool},\text{route},\text{resource},
\text{effect},\text{policy},\text{risk},\text{uncertainty},\text{context},
\text{evidence},\text{privacy}\}.
\]
Let $\phi_s(t,x,\facts)$ map the proposed action and server state to independently held facts over the same dimensions.
\end{definition}

The implementation refines these dimensions into 18 stable packet predicates:
schema validity; tool identity; route membership; intent-class equality;
operation equality and intent confinement; resource identity; effect bounds;
destination; payload confinement; policy-basis scope/rule and disposition; risk;
uncertainty; context digest; evidence completeness and freshness; and raw
context disclosure. Route is therefore explicit rather than silently folded
into tool identity, while evidence completeness and freshness refine the
evidence dimension.

Authoritative availability is checked per field or predicate dependency rather
than with one whole-object gate. A missing or malformed server operand marks
only its dependent predicate or predicates unknown; a contradiction on an
independent available operand is still hard. Resource type and resource ID
availability, policy scope and rule availability, and each required evidence
digest are represented separately. Registered authoritative intent classes,
operations, and context source types must also belong to the public profile
taxonomy; an otherwise bounded but unregistered server value is unavailable,
not a mismatch.

For an optional authoritative effect fact, the verifier uses three states. A
present valid value is known independently of malformed envelope siblings; a
present invalid value is unknown; and an absent value is known absent only
when the corresponding operation/resource envelope is complete. Thus
malformed siblings neither fabricate inapplicability nor mask a contradiction
on a present valid bound. Record bounds are integers in $[0,10^6]$ and amount
bounds are finite numbers in $[0,10^{12}]$, matching the public profile.

Exact-set comparisons slice at the registered field cap before filtering and
deduplicating parseable bounded strings. The packet serializes exactly the
same canonical operands. A duplicate matching value, invalid extra, or
beyond-cap tail therefore cannot create a packet-invisible semantic result,
while a parseable in-cap outside or missing value cannot be hidden by a
malformed sibling. The claim-effect normalizer accepts only public claim field
names; payload-only names such as \texttt{resourceId}, \texttt{limit},
\texttt{amount}, and \texttt{fields} are never aliases for claim fields.
Evidence freshness is likewise independent of one-to-one completeness: any
format-valid required reference that differs from its available current digest
is hard, even if another required type is missing or duplicated. Authoritative
evidence is projected in the fixed required-type order; extra keys make overall
evidence completeness unavailable but cannot evict a known current digest from
the packet.

For dimension $d$, define a comparison function

\[
v_d(e,t,x,\facts) \in \{\mcode{match},\mcode{soft},\mcode{hard},\mcode{unknown}\}.
\]

For an applicable numeric effect dimension $j$, the reference profile accepts
the declared bound only when
\[
b_j(x)\leq b_j(e)\leq b_j(C).
\]
For field sets it requires
\[
F(x)\subseteq F(e)\subseteq F(C).
\]
A substantive declaration on a payload-inapplicable numeric, field, or
destination dimension is hard. Context pairs are compared after set
normalization:
\[
\mathsf{Deps}(e)=
\begin{cases}
\varnothing,&\neg\mathsf{Untrusted}(X_s),\\
\{(\mathsf{source}(X_s),\mathsf{digest}(X_s))\},
  &\mathsf{Untrusted}(X_s).
\end{cases}
\]
Unavailable authoritative operands yield unknown rather than match. For
required evidence type $k$, any format-valid claimed digest
$h\in H_e(k)$ with available current $h_s(k)$ and $h\neq h_s(k)$ yields hard
freshness independently of whether $|H_e(k)|=1$ for every required type.

A \emph{hard} result denotes contradiction with an execution-critical fact, including tool identity, operation, target resource, effect bound, external destination, material context omission, or raw-context disclosure. A \emph{soft} result denotes uncertainty, an authoritative review requirement, incomplete policy evidence, risk understatement, or another condition that requires review but does not by itself prove malicious intent. Unknown evidence cannot be interpreted as a match.

Let the decision domain be the ordered set
\[
\mathcal{D}_3=(\{\allow,\review,\deny\},\preceq),
\qquad
\allow \prec \review \prec \deny .
\]
Map predicate outcomes and the server-owned tool-risk requirement into this
domain:
\[
s(\mcode{match})=\allow,\quad
s(\mcode{soft})=s(\mcode{unknown})=\review,\quad
s(\mcode{hard})=\deny,
\]
\[
r(t)=
\begin{cases}
\review,&\mathsf{HighRisk}(t),\\
\allow,&\text{otherwise}.
\end{cases}
\]
The action-claim disposition is the least upper bound
\[
\decision(e,t,x,\facts)
=
\left(\bigsqcup_{d\in D}s\!\left(v_d(e,t,x,\facts)\right)\right)
\sqcup r(t).
\]
Thus any hard predicate dominates; absent a hard predicate, soft, unknown, or
high-risk evidence yields review; only all-match, non-high-risk claims allow.

Let $d_0(u,t,x,\facts)$ denote current baseline authorization and routing.
Let $d_E=\decision(e,t,x,\facts)$ denote the action-claim disposition, and let
$d_X(t,x,\facts)$ denote current effect control. All three belong to
$\mathcal{D}_3$. Their composition is the least upper bound
\[
d^\star = d_0 \sqcup d_E \sqcup d_X = \max_{\preceq}\{d_0,d_E,d_X\}.
\]
An effect may execute only when $d^\star=\allow$. A \review{} result creates a non-executing draft, preflight, or confirmation object. Approval of that object is a new decision event: all authoritative facts must be resolved again and $d^\star$ recomputed at action time.

The reference library exports this same three-operand most-restrictive join,
and the conformance runner enumerates all \JoinCompositionChecks{} possible
triples in $\mathcal{D}_3^3$. This validates the finite join implementation,
not complete mediation by the draft-only wrapper: that study observes $d_E$
before a forced-draft boundary and does not instantiate an effect executor.

\begin{proposition}[Design invariant: non-widening composition]
For all inputs, $d^\star \succeq d_0$ and $d^\star \succeq d_X$. The action-claim check therefore cannot make either baseline authorization or effect control less restrictive.
\end{proposition}
\begin{proof}
$d^\star$ is the least upper bound under $\preceq$, so it is at least every operand. No explanation field occurs in an override path outside the join.
\end{proof}

\begin{proposition}[Design invariant: hard-contradiction non-execution]
Assume complete mediation, sound hard-comparison predicates, correct and current server facts, and action-time recomputation. If an execution-critical dimension returns $\mcode{hard}$, the associated effect does not execute.
\end{proposition}
\begin{proof}
A hard result yields $d_E=\deny$, hence $d^\star=\deny$. Complete mediation permits execution only at $\allow$. Action-time recomputation extends the argument to a proposal that previously entered review.
\end{proof}

\begin{proposition}[Design invariant: decision monotonicity]
Holding all other facts fixed, replacing a matched dimension with soft, unknown, or hard evidence cannot produce a less restrictive $d_E$ or $d^\star$.
\end{proposition}
\begin{proof}
A soft or unknown dimension removes the all-match condition and yields at least review; a hard dimension yields deny. The join operator is monotone in each operand.
\end{proof}

\begin{proposition}[Design invariant: conditional structural replay]
Let the ordinary audit projection contain the profile version, action-claim disposition $d_E$ and reason categories, schema-error categories, predicate identifiers and outcomes, bounded normalized claim and authoritative operands, and the current intent, policy, payload, and applicable context digests. If the corresponding authoritative facts remain available through authorized digest resolution, those elements suffice to replay the reference-profile rule evaluation without storing raw untrusted content in the ordinary packet.
\end{proposition}
\begin{proof}
The deterministic predicates consume the recorded normalized operands and resolved facts rather than source prose. The projection records the operands, categorical predicate results, and digest references used by the rule table. Re-evaluating that table therefore recovers the recorded disposition, conditional on authorized resolution of the same facts. This invariant concerns profile-level rule replay, not semantic secret detection, source retention, principal or tenant binding, or whether a human can understand the action without separately authorized source access.
\end{proof}

These design invariants hold under the stated assumptions. Their scope excludes the correctness of intent classification, route membership, risk labels, policy facts, source digests, reviewer decisions, and the later action-time transition. The reference packet omits a principal--tenant--session--action envelope and an independent snapshot-version field; cross-domain replay therefore requires those values to be added and integrity-bound.

\section{Action-Claim Mediation Architecture}

\begin{figure}[h]
\centering
\begin{tikzpicture}[
  node distance=0.8cm and 0.7cm,
  box/.style={draw, rounded corners, align=center, minimum height=0.78cm, text width=0.18\linewidth, font=\small},
  arrow/.style={-{Latex[length=2mm]}, thick}
]
\node[box, fill=opBlueFill, draw=opBlue] (proposal) {Tool proposal\\$t,x,e$};
\node[box, fill=opPurpleFill, draw=opPurple, right=of proposal] (facts) {Server facts\\$C,T,P,R,X_s,Q$};
\node[box, fill=opAmberFill, draw=opAmber, below=of $(proposal)!0.5!(facts)$] (verify) {Explanation contract\\parser and verifier};
\node[box, fill=opGreenFill, draw=opGreen, below left=of verify] (allow) {Allow\\governed action path};
\node[box, fill=opAmberFill, draw=opAmber, below=of verify] (review) {Review\\draft / preflight / confirm};
\node[box, fill=opRedFill, draw=opRed, below right=of verify] (deny) {Deny\\stable reason code};
\node[box, fill=black!3, below=1.0cm of review, text width=0.48\linewidth] (audit) {Minimal-disclosure evidence: normalized fields, predicate outcomes, decision, reason categories, and evidence digests};
\draw[arrow] (proposal) -- (verify);
\draw[arrow] (facts) -- (verify);
\draw[arrow] (verify) -- (allow);
\draw[arrow] (verify) -- (review);
\draw[arrow] (verify) -- (deny);
\draw[arrow] (allow) -- (audit);
\draw[arrow] (review) -- (audit);
\draw[arrow] (deny) -- (audit);
\end{tikzpicture}
\caption{The mediation layer verifies an untrusted explanation against independently held facts. Every outcome produces minimized decision evidence.}
\ifEBTEIEEEBuild\else
  \Description{A tool proposal and authoritative server facts enter an explanation-contract verifier. The verifier routes the request to allow, non-executing review, or deny, and every branch emits minimized audit evidence.}
\fi
\label{fig:architecture}
\end{figure}
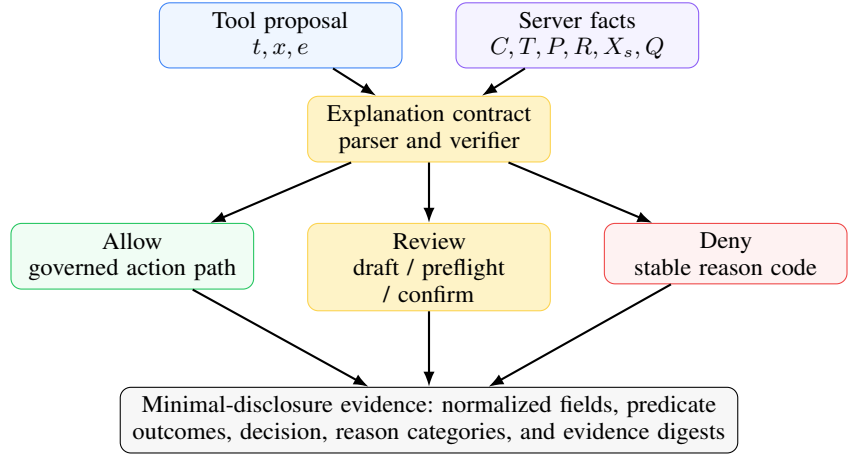

\subsection{Processing Stages}

The reference design has six stages:

\begin{enumerate}
  \item \textbf{Parse and bound.} Reject unknown schema versions, excessive field sizes, raw context, and malformed evidence references.
  \item \textbf{Resolve current facts.} Load the current intent, tool, policy, route, context, and preflight records under the authenticated application and actor.
  \item \textbf{Normalize.} Map explanation claims and server facts to a common operation-resource-effect representation.
  \item \textbf{Compare.} Evaluate every contract dimension with stable categorical results.
  \item \textbf{Route.} Apply the hard/soft decision rules and then the existing risk-sensitive action pipeline.
  \item \textbf{Record.} Store minimized, integrity-bound decision evidence and exclude raw untrusted content.
\end{enumerate}

\subsection{Stable Reason Categories}

Hard categories cover tool identity, target resources, effect bounds, omitted untrusted dependencies, and raw-context disclosure. Soft categories cover risk understatement, high uncertainty, and missing policy evidence. Stable external codes identify the failed contract dimension while omitting source content and system-topology detail.

\section{Composition with Agent Governance}

The action-claim check is the fourth layer in a compositional governance path:

\[
\text{intent} \rightarrow \text{context authority} \rightarrow \text{safe route} \rightarrow \text{explanation} \rightarrow \text{effect controls}.
\]

Intent-governed authorization narrows capabilities to the user's request~\cite{zhu2026igac}. Context-sensitive permission shrinkage reduces authority when untrusted content is present. Safe routing confines model selection to an eligible candidate set. The action-claim layer then checks whether the claimed reason and expected effect agree with those upstream records and the actual payload. Finally, preflight, draft, confirmation, idempotency, and state-witness controls govern side effects.

The layers are not interchangeable. A correct explanation does not repair an unauthorized tool. A safe route does not prove payload bounds. A valid intent certificate does not show that all untrusted context dependencies were disclosed. Conversely, an explanation contradiction can provide an additional deny or review signal even when the underlying credential is broad.

\section{Reference Artifact and Evaluation Method}

\subsection{Research Questions}

\begin{itemize}
  \item \textbf{RQ1---Contract and conformance:} Which action claims and authoritative predicates constitute the reference profile, and does the implementation conform on the included mutations and metamorphic relations?
  \item \textbf{RQ2---Draft-only integration:} Can the verifier mediate a reference intent-to-draft path without forwarding authored hard contradictions, while preserving aligned and soft-review draft paths?
  \item \textbf{RQ3---Generation behavior:} What parse, schema, disposition, provider, and repair failure modes appear when four hosted model families emit the bounded contract?
  \item \textbf{RQ4---External semantic transfer:} On a pinned external task structure, what additional diagnostic distinction does authoritative task--proposal binding provide beyond the existing high-risk non-allow gate?
\end{itemize}

\subsection{Conditions}

We compare the full profile with three predicate ablations:

\begin{enumerate}
  \item \textbf{Free-form justification:} any nonempty summary satisfies the explanation requirement; high-risk tools still route to review.
  \item \textbf{Structured schema only:} required fields must be present, but their values are not compared with server facts.
  \item \textbf{Payload-bound explanation:} tool, operation, resource, and numeric payload bounds are compared; policy, risk, provenance, and privacy checks are omitted.
  \item \textbf{Full profile:} the reference-profile contract dimensions are checked, hard and soft contradictions are separated, and ordinary audit evidence is minimized.
\end{enumerate}

Expected dispositions and the full evaluator instantiate the same public profile. The study therefore measures implementation conformance and the incremental coverage of named predicates over the authored fixture set.

\subsection{Task Construction}

The suite contains \TaskFamilies{} abstract task families: bounded record summarization, bounded expense creation, record update, record export, record deletion, appointment creation, agent-version publication, and knowledge search. All identifiers and payloads are synthetic.

The evidence snapshot represented by the preamble macros contains one aligned scenario and \VariantsPerFamily{} total scenarios per family, including schema, tool, authorized-route, intent-class, operation, resource, effect-bound, policy-rule, context, disclosure, evidence-freshness, and evidence-completeness mutations. This yields \ConformanceTasks{} tasks and \ConformanceRows{} task--configuration rows. Expected outcomes follow the profile definition: aligned low/medium-risk cases allow, aligned high-risk cases review, hard contradictions deny, and soft or incomplete evidence reviews. A payload-outside-intent case keeps the explanation consistent with the payload so that only a verifier resolving the authoritative intent boundary can reject it. Table~\ref{tab:fieldauthority} records the source and tested status of each predicate.

\subsection{Metrics}

We report profile-disposition agreement (PDA), specified non-allow coverage (SNAC), hard-fixture allow rate (HFAR), soft-fixture allow rate (SFAR), aligned-fixture agreement (AFA), review-routing fraction (RRF), fact-reference resolvability (FRR), and explicit-fixture leak acceptance (EFLA). PDA measures agreement with the authored profile, and SNAC measures the fraction of designated non-aligned fixtures routed non-allow. RRF summarizes the deliberately enriched fixture mixture. FRR requires a valid bounded schema, exactly one current intent, policy, and payload digest, and a syntactically valid digest for every declared context dependency. These rates describe the complete finite fixture set.

The artifact also runs \MetamorphicChecks{} metamorphic checks. For every task, the full profile must be at least as restrictive as the payload-bound configuration. For each registered hard mutation, the changed case must not yield allow or a less restrictive outcome than the aligned case.

\FloatBarrier
\section{Deterministic Results}

\begin{table}[h]
\centering
\scriptsize
\resizebox{\linewidth}{!}{%
\begin{tabular}{L{0.22\linewidth}rrrrrrrr}
\toprule
\textbf{Configuration} & \textbf{PDA} & \textbf{SNAC} & \textbf{HFAR} & \textbf{SFAR} & \textbf{AFA} & \textbf{RRF} & \textbf{FRR} & \textbf{EFLA} \\
\midrule
Free-form justification & \FreePDA{} & \FreeSNAC{} & \FreeHFAR{} & \FreeSFAR{} & \FreeAFA{} & \FreeRRF{} & \FreeSEC{} & \FreeEFLA{} \\
Structured schema only & \SchemaPDA{} & \SchemaSNAC{} & \SchemaHFAR{} & \SchemaSFAR{} & \SchemaAFA{} & \SchemaRRF{} & \SchemaSEC{} & \SchemaEFLA{} \\
Payload-bound explanation & \PayloadPDA{} & \PayloadSNAC{} & \PayloadHFAR{} & \PayloadSFAR{} & \PayloadAFA{} & \PayloadRRF{} & \PayloadSEC{} & \PayloadEFLA{} \\
Full profile & \FullPDA{} & \FullSNAC{} & \FullHFAR{} & \FullSFAR{} & \FullAFA{} & \FullRRF{} & \FullSEC{} & \FullEFLA{} \\
\bottomrule
\end{tabular}}
\caption{Authored profile-conformance and predicate-ablation results over \ConformanceTasks{} tasks per configuration.}
\label{tab:results}
\end{table}
\FloatBarrier

\subsection{RQ1: Predicate-Ablation Conformance}

Free-form justification applies only the review requirement already attached to high-risk tools and has HFAR \FreeHFAR{} and SFAR \FreeSFAR{} in this snapshot. Strict schema validation routes malformed and extra-field disclosure fixtures non-allow, yielding SNAC \SchemaSNAC{}, HFAR \SchemaHFAR{}, and SFAR \SchemaSFAR{}. Payload binding activates the authored tool, resource, operation, and numeric predicates, yielding SNAC \PayloadSNAC{} and HFAR \PayloadHFAR{}. It lacks the full profile's authoritative intent, route, context, freshness, risk, and policy comparisons.

The full profile returns non-allow for every registered non-aligned fixture in the current snapshot. It denies all \HardFixtures{} designated hard cases and routes all \SoftFixtures{} soft cases to review, matching the implemented rule table over the included scenarios.

\subsection{RQ1: Aligned Paths and Review Routing}

All four configurations preserve the authored disposition for all \AlignedFixtures{} aligned tasks. \AlignedAllows{} low- or medium-risk aligned tasks allow, and \AlignedReviews{} policy- or risk-sensitive aligned tasks review in the current snapshot. The full profile's RRF is \FullRRF{} because the suite deliberately includes those aligned review paths and all soft fixtures. Other configurations route different authored mutations to review as a consequence of their active predicates.

\subsection{RQ1: Structural Packets, Fact References, and Disclosure Fixtures}

The full profile emits a versioned structural packet for all \StructuralPacketChecks{} full-profile outcomes. The conformance checker verifies packet identity, decision and reason equality, all 18 stable predicate identifiers and categorical outcomes, per-operand authoritative availability, cap-identical normalized claim/action/fact operands, the three current authoritative digests in fixed type order, and exclusion of both free-text summaries and the fixed raw marker. It also records bounded Boolean claim diagnostics for raw disclosure, duplicate evidence type, missing policy basis, and schema-error truncation; these distinguish special schema dispositions without retaining the offending content. Predicate aggregation uses the total order
\(\mcode{hard}>\mcode{unknown}>\mcode{soft}>\mcode{match}\), independent of reason order. FRR is separately \FullSEC{} because stale, duplicate, or schema-invalid claims do not have a complete current claim-to-fact reference set even though their non-allow packet remains structurally well formed. The ablations emit no full-profile packet and have FRR 0. Strict schema validation and the full profile reject every included explicit \texttt{rawContent} fixture, and the artifact scanner reports no fixed synthetic marker in generated evidence rows.

\subsection{RQ1: Metamorphic Regression Checks}

All \MetamorphicChecks{} registered checks pass in the current snapshot. The full profile is never less restrictive than payload binding for the same task, and each designated hard mutation is non-allow and at least as restrictive as its aligned counterpart. These checks provide regression evidence for the stated decision order.

A separate control-flow regression applies an authoritative policy denial to
each task family. All \AuthoritativePolicyDenyChecks{} schema-valid claims and
all \AuthoritativePolicyMalformedDenyChecks{} malformed claims remain deny.
Thus, a generic schema-review route cannot downgrade an independently resolved
server denial. These auxiliary checks are outside the \ConformanceTasks{}
profile-task denominator and are reported separately. The same runner keeps
all \AuthoritativePolicyReviewChecks{} matching authoritative-review cases at
review, all \SchemaSoftHardDominanceChecks{} mixed schema-soft plus semantic-hard
cases at deny, all \RawSchemaHardDominanceChecks{} raw-schema-deny plus
semantic-hard cases with both hard outcomes retained, and all
\HardSoftPacketTruthChecks{} hard-plus-soft packet cases truthful on both
evaluated predicates. It also routes all
\EffectClaimOmissionChecks{} quantitative-bound,
\DestinationClaimOmissionChecks{} destination, and
\FieldClaimOmissionChecks{} field-set omissions to review while marking the
corresponding packet predicate unknown.

Compound evidence controls keep all \StaleIncompleteEvidenceChecks{}
stale-plus-duplicate cases at deny with freshness hard and completeness
unknown, while all \UnavailableAuthoritativeEvidenceChecks{} unavailable
current-digest cases remain review with both evidence predicates unknown.
Mixed-array controls distinguish \MixedMalformedAlignedSetChecks{} aligned
duplicate/invalid-extra cases, which remain schema review without a fabricated
semantic hard result, from \MixedMalformedKnownHardChecks{} cases containing a
parseable outside value, which remain deny. All
\SelfDenialMalformedChecks{} malformed sibling-field cases retain policy
self-denial, all \TypedRawDisclosureChecks{} object/array raw-field cases deny,
and the \MalformedExpectedEffectChecks{} malformed-effect family pairs return
review alone or deny when an independent hard mismatch is present.

Effect controls deny all \EffectUpperBoundChecks{} claims that exceed intent
and all \InapplicableEffectClaimChecks{} substantive claims on inapplicable
dimensions, while preserving \EffectSafeSlackChecks{} bounds satisfying
$\text{payload}\leq\text{claim}\leq\text{intent}$. Context controls deny all
\ContextExactnessHardChecks{} trusted-context or extra-pair mismatches,
preserve \DuplicateCurrentContextChecks{} duplicate-current pairs after set
normalization, and route all \UnavailableAuthoritativeContextChecks{}
unavailable current states to review with context unknown. Finally,
\AuthoritativeFactAvailabilityChecks{} field-scoped availability controls and
\OrthogonalAuthoritativeMalformedHardChecks{} malformed-fact plus independent
hard controls verify that unavailable operands cannot appear as matches or
mask safely resolved contradictions. The
\EmptyAuthoritativeObjectChecks{} empty intent/payload object pairs,
\MissingPayloadEnvelopeSiblingChecks{} missing payload-envelope siblings, and
\MissingIntentEnvelopeSiblingChecks{} missing intent-envelope siblings ensure
that optional absence is not misread as known inapplicability until the
corresponding authoritative envelope is available. Conversely,
\PresentOptionalOrthogonalHardChecks{} present-optional controls retain known
record, amount, field, and destination contradictions despite malformed
siblings; \ResourceOperandAvailabilityChecks{} controls independently preserve
resource-type and resource-ID results; and
\InvalidAuthoritativeNumericDomainChecks{} out-of-domain server bounds remain
unknown.

Policy and evidence controls independently cover
\SelfDenialUnavailableAuthorityChecks{} self-denials with unavailable current
policy dispositions, \PolicyBasisOperandAvailabilityChecks{} malformed scopes
with known rule mismatches, \ScopeFreePolicyBasisChecks{} valid empty/empty
scope sets, \EmptyScopeMissingChecks{} empty/nonempty scope differences, and
\EvidenceExtraKeyStaleChecks{} extra authoritative evidence keys combined with
known stale required digests. Canonical-replay controls include
\CanonicalScopeClaimChecks{} scope-tail shapes,
\WideAllowedFieldReplayChecks{} forty-field pair,
\OverCapAuthoritativeSetChecks{} over-cap route/resource family pairs,
\TailBoundedClaimChecks{} capped-claim family suites,
\TailAllowedFieldClaimChecks{} allowed-field tail,
\TailContextDependencyChecks{} context tails,
\CanonicalInvalidScalarChecks{} invalid scalar families, and
\InvalidAuthoritativeContextSourceChecks{} invalid current-source cases.

Finally, \ClaimAliasIsolationChecks{} claim-alias pairs,
\AuthoritativeEnumAvailabilityChecks{} authoritative-taxonomy controls,
\PacketDiagnosticCollisionChecks{} diagnostic collision pairs, and
\PacketOutcomeOrderChecks{} reason-order checks exercise the packet trust
boundary. Schema validation adds \UnicodeLengthChecks{} Unicode code-point
boundaries and \BoundedValidationResourceChecks{} oversized-array/error-cap
controls. These parsed-object checks do not replace the profile requirement
that an integration enforce the 1\,MiB UTF-8 transport limit before JSON
parsing. The runner also exhaustively verifies the
\JoinCompositionChecks{} baseline--action-claim--effect decision triples.

\FloatBarrier
\section{Draft-Only Reference Integration}

\subsection{Design}

To answer RQ2, we compose the reusable verifier with an OpenPort reference implementation through in-process HTTP injection at the intent and draft endpoints. \RuntimeFamilies{} task families map to implemented action tools: transaction creation, update, soft deletion, and bounded export. The evaluated snapshot contains \RuntimeTasks{} tasks and \RuntimeRows{} task--configuration rows.

For every task, the harness creates an intent certificate through the reference intent endpoint. If the explanation configuration returns allow or review, the wrapper submits a synthetic payload to the reference draft endpoint and binds the intent certificate. If the configuration returns deny, the draft endpoint is not called. Each row uses a fresh in-memory runtime, and every forwarded action terminates as a non-executing draft.

\subsection{Results}

\begin{table}[h]
\centering
\scriptsize
\resizebox{\linewidth}{!}{%
\begin{tabular}{L{0.22\linewidth}rrrrrrrr}
\toprule
\textbf{Configuration} & \textbf{HCFD} & \textbf{Soft draft} & \textbf{Aligned draft} & \textbf{Audit bind} & \textbf{FRR} & \textbf{Exec.} & \textbf{p95 verify ms} & \textbf{p95 total ms} \\
\midrule
Free-form justification & \RuntimeFreeHCFD{} & \RuntimeFreeSoftDraft{} & \RuntimeFreeAlignedDraft{} & \RuntimeFreeAudit{} & \RuntimeFreeSEC{} & \RuntimeFreeExec{} & \RuntimeFreeVerifyP{} & \RuntimeFreeTotalP{} \\
Structured schema only & \RuntimeSchemaHCFD{} & \RuntimeSchemaSoftDraft{} & \RuntimeSchemaAlignedDraft{} & \RuntimeSchemaAudit{} & \RuntimeSchemaSEC{} & \RuntimeSchemaExec{} & \RuntimeSchemaVerifyP{} & \RuntimeSchemaTotalP{} \\
Payload-bound explanation & \RuntimePayloadHCFD{} & \RuntimePayloadSoftDraft{} & \RuntimePayloadAlignedDraft{} & \RuntimePayloadAudit{} & \RuntimePayloadSEC{} & \RuntimePayloadExec{} & \RuntimePayloadVerifyP{} & \RuntimePayloadTotalP{} \\
Full profile & \RuntimeFullHCFD{} & \RuntimeFullSoftDraft{} & \RuntimeFullAlignedDraft{} & \RuntimeFullAudit{} & \RuntimeFullSEC{} & \RuntimeFullExec{} & \RuntimeFullVerifyP{} & \RuntimeFullTotalP{} \\
\bottomrule
\end{tabular}}
\caption{Local draft-only integration results over \RuntimeTasks{} tasks per configuration. HCFD is the fraction of \RuntimeHardFixtures{} designated hard cases forwarded to the draft endpoint. All forwarded requests are forced to draft; latency is a harness diagnostic.}
\label{tab:runtime}
\end{table}
\FloatBarrier

Free-form justification forwards every designated hard case to the draft endpoint. Strict schema validation blocks malformed cases but forwards hard semantic mismatches outside its scope, yielding HCFD \RuntimeSchemaHCFD{}. Payload binding blocks direct tool, resource, and effect mismatches but forwards cases governed by omitted predicates, yielding HCFD \RuntimePayloadHCFD{}. The full profile forwards none of \RuntimeHardFixtures{} designated hard cases, while all \RuntimeSoftFixtures{} soft and \RuntimeAlignedFixtures{} aligned cases create governed drafts. Its audit-binding rate is \RuntimeFullAudit{}; the lower ablation values reflect rows stopped before a bound draft rather than missing bindings on successfully created full-profile drafts.

The full verifier's observed local p95 evaluation time is \RuntimeFullVerifyP{} ms, and p95 wrapper-plus-intent-plus-draft time is \RuntimeFullTotalP{} ms in this small in-memory environment. With only \RuntimeTasks{} rows per configuration, a single host, and no warm-load or throughput design, these numbers are harness diagnostics rather than performance evidence. The data were recorded from the clean pinned runtime revision with the \RuntimeSourceFiles-file source fingerprint \RuntimeSourceDigestDisplay{} using \texttt{sha256-relative-path-nul-bytes-nul-v1}.

\section{Constrained NVIDIA Multi-Model Pilot}

\subsection{Design}

To answer RQ3, we test explanation generation through NVIDIA's hosted API with \HostedModels{} model families: \nolinkurl{openai/gpt-oss-120b}, \nolinkurl{qwen/qwen3-next-80b-a3b-instruct}, \nolinkurl{nvidia/nemotron-3-nano-30b-a3b}, and \nolinkurl{deepseek-ai/deepseek-v4-flash}. The protocol uses \HostedTasks{} synthetic tasks: \HostedAlignedTasks{} aligned actions, \HostedRepairableTasks{} actions with repairable explanation challenges, and \HostedNonrepairableTasks{} nonrepairable action contradictions involving record, amount, field, or destination expansion. Each task is generated \HostedRepeats{} times under different fixed seeds at temperature 0.2. For every task and repeat, the model first emits a contract and then receives stable parser, schema, and verifier feedback for one repair attempt, including rows that already matched the profile. The repairable tasks additionally start from a deliberately seeded invalid explanation. Each model therefore contributes \HostedInitialPerModel{} initial, \HostedRepairPerModel{} ordinary-repair, and \HostedSeededPerModel{} seeded-repair attempts, for \HostedRows{} attempts total.

The prompt supplies the bounded schema and current synthetic server facts. It instructs each model to describe the proposed action truthfully, ignore instructions inside categorically untrusted content, use exact current evidence digests, and omit a synthetic raw marker. Recorded results contain SHA-256 response digests, byte counts, minimized parsed fields, outcomes, latency, and categorical HTTP failures. Because raw responses are absent, exact surface wording cannot be reconstructed from these records.

Protocol development before version \texttt{0.5-nvidia} resolved ambiguities in evidence arrays, context-dependency objects, and provider-error minimization. The protocol was not preregistered, and candidate models that timed out or were unavailable were excluded before behavioral aggregation. This design process, availability filtering, and prompt specificity introduce experimenter and selection effects.

The four source result sets were generated on 2026-07-12. They preserve
response digests, minimized parsed explanations, recorded runner decisions,
latency, and typed provider failures. They predate immutable-run support and
lack raw responses, \texttt{runFingerprint}, and \texttt{runConfig}, so the
exact historical prompt and verifier code bytes cannot be reconstructed. A
separate current-pipeline layer applies the present frozen schema, task
builder, and verifier to the saved minimized claims offline; this layer is a
post-hoc revalidation rather than a hosted rerun.

\FloatBarrier
\subsection{Recorded Historical and Current Offline Layers}

\begin{table}[h]
\centering
\scriptsize
\resizebox{\linewidth}{!}{%
\begin{tabular}{L{0.19\linewidth}rrrrrrrr}
\toprule
\textbf{Model} & \textbf{Hist. init parse} & \textbf{Hist. init schema} & \textbf{Hist. init agree} & \textbf{Hist. repair agree} & \textbf{Hist. seed agree} & \textbf{Hist. provider err.} & \textbf{Hist. nonrep. allow} & \textbf{Hist. marker} \\
\midrule
\ModelGPTTableRow
\ModelQwenTableRow
\ModelNemotronTableRow
\ModelDeepSeekTableRow
\bottomrule
\end{tabular}}
\caption{Frozen 2026-07-12 recorded hosted-generation/runner layer. Agreement excludes provider failures from the numerator while retaining all attempted calls in the primary denominator; parentheses give conditional-on-response values where they differ. These are historical runner decisions, not current-verifier dispositions.}
\label{tab:modelpilot}
\end{table}
\FloatBarrier

\begin{table}[h]
\centering
\scriptsize
\resizebox{\linewidth}{!}{%
\begin{tabular}{L{0.15\linewidth}rrrrrrrrrr}
\toprule
\textbf{Layer} & \textbf{Attempts} & \textbf{Responses} & \textbf{Provider fail} & \textbf{Saved claims} & \textbf{Parse fail} & \textbf{Hist. conformant} & \textbf{Init agree} & \textbf{Repair agree} & \textbf{Seed agree} & \textbf{Current A/R/D} \\
\midrule
Current post-hoc offline & \HostedCurrentAttempts{} & \HostedCurrentResponses{} & \HostedCurrentProviderFailures{} & \HostedCurrentSavedClaims{} & \HostedCurrentParseFailures{} & \HostedHistoricalConformantRows{} & \HostedCurrentInitialAgreementAll{} (\HostedCurrentInitialAgreementResponse{}) & \HostedCurrentRepairAgreementAll{} (\HostedCurrentRepairAgreementResponse{}) & \HostedCurrentSeededAgreementAll{} (\HostedCurrentSeededAgreementResponse{}) & \HostedCurrentAllows{}/\HostedCurrentReviews{}/\HostedCurrentDenies{} \\
\bottomrule
\end{tabular}}
\caption{Current-pipeline post-hoc offline revalidation of the immutable minimized 2026-07-12 source layer. Parentheses give conditional-on-response agreement. It made zero model, provider, or network calls and is not a new generation run.}
\label{tab:modelrevalidation}
\end{table}
\FloatBarrier

Across the frozen recorded calls, historical generated-output/runner agreement excluding provider failures is \HostedInitialAgreementAll{} initially, \HostedRepairAgreementAll{} after ordinary feedback, and \HostedSeededAgreementAll{} for seeded-defect repair. Conditional on receiving a provider response, the corresponding historical values are \HostedInitialAgreementResponse{}, \HostedRepairAgreementResponse{}, and \HostedSeededAgreementResponse{}. Provider availability is therefore reported separately from recorded generation agreement rather than being credited when a transport failure happened to route to the authored review disposition. Variation in this historical layer is substantial: GPT-OSS and Qwen agree on every initial attempt, Nemotron has frequent parse/schema failures, and DeepSeek agrees on every successful response but has substantial provider failure.

In the recorded historical layer, among the \HostedRepairRows{} paired initial/ordinary-repair attempts, \HostedCorrectToCorrect{} remain agreeing, \HostedCorrectToIncorrect{} move from agreement to disagreement, \HostedIncorrectToCorrect{} move from disagreement to agreement, and \HostedIncorrectToIncorrect{} remain disagreeing. Historical feedback is therefore not monotonically associated with recorded profile agreement. For the \HostedNonrepairableRows{} historical post-feedback attempts involving nonrepairable actions, \HostedNonrepairableDeny{} return deny and \HostedNonrepairableReview{} return non-executing review; none returns allow. Across all \HostedExpectedDenyRows{} recorded expected-deny initial/repair attempts, the allow count is also zero. No recorded output contains the fixed synthetic marker (0/\HostedRows{}). The \HostedProviderErrors{}/\HostedRows{} historical provider failures are retained as attempted calls and route to review.

The current offline layer evaluates all \HostedCurrentAttempts{} immutable
attempt records: \HostedCurrentResponses{} retained provider responses,
\HostedCurrentProviderFailures{} typed provider failures,
\HostedCurrentSavedClaims{} saved parsed explanations, and
\HostedCurrentParseFailures{} response parse failures. Its
\HostedHistoricalConformantRows{} historically schema-conformant rows pass the
current task-fact evidence/context consistency precheck. Current disposition
agreement is \HostedCurrentInitialAgreementAll{},
\HostedCurrentRepairAgreementAll{}, and
\HostedCurrentSeededAgreementAll{} over all corresponding attempts, or
\HostedCurrentInitialAgreementResponse{},
\HostedCurrentRepairAgreementResponse{}, and
\HostedCurrentSeededAgreementResponse{} conditional on a provider response.
The resulting current operational dispositions are
\HostedCurrentAllows{} allow, \HostedCurrentReviews{} review, and
\HostedCurrentDenies{} deny. Within the \HostedNonrepairableRows{}
nonrepairable post-feedback records, the current pipeline returns
\HostedCurrentNonrepairableDeny{} deny,
\HostedCurrentNonrepairableReview{} review, and
\HostedCurrentNonrepairableAllow{} allow. Across the
\HostedExpectedDenyRows{} current expected-deny initial/repair records, the
allow count is \HostedCurrentExpectedDenyAllow{}.

Relative to the recorded historical runner disposition,
\HostedCurrentDecisionChanges{} received-response rows change:
\HostedCurrentAllowToReview{} allow-to-review,
\HostedCurrentAllowToDeny{} allow-to-deny, and
\HostedCurrentReviewToDeny{} review-to-deny. There are
\HostedCurrentDenyDowngrades{} deny-to-review or deny-to-allow transitions.
These are verifier-version transitions, not changes in model behavior. The
derived layer is bound to provenance fingerprint
\HostedCurrentFingerprintDisplay{} and made
\HostedCurrentProviderCalls{} network, model, or provider calls. Because the
historical sources lack \texttt{runFingerprint}, \texttt{runConfig}, raw
responses, and exact historical prompt/code bytes, the current layer cannot
retroactively authenticate or reproduce the original online configuration.

The historical layer records cross-model interface behavior on the
fact-complete suite, while the current layer records post-hoc dispositions
recoverable from the minimized saved fields. The reported scope is limited to
fail-closed routing, generation failures, and repair transitions in these
records.

\FloatBarrier
\section{Pinned AgentDojo Semantic Adaptation}

\subsection{Design and Provenance}

To answer RQ4, we reuse a minimized proposal inventory derived from a pinned AgentDojo~\cite{debenedetti2024agentdojo} banking-suite semantic adapter: release \texttt{0.1.35}, clean revision \AgentDojoRevisionDisplay, and benchmark \texttt{v1.2.2}. The source inventory contains \ExternalProposals{} proposal records: \ExternalBenign{} calls from \ExternalOrdinaryTasks{} ordinary user tasks, \ExternalDirect{} calls representing legitimate direct user goals, and \ExternalAttacks{} cross-task attack proposals formed by pairing each ordinary task with each direct-injection goal. The adaptation retains upstream task identifiers, paired-task relations, and public tool names.

For every record, the proposed explanation truthfully matches the proposed tool and payload. Ordinary and legitimate direct-goal proposals bind the payload to the current task identifier. Cross-task attacks instead bind the payload to the paired injection-task identifier while the authoritative intent certificate remains scoped to the ordinary user task. Untrusted-context evidence is declared and current, so task--proposal inconsistency is the only intended hard contradiction. The same four deterministic configurations process all \ExternalProposals{} proposals, producing \ExternalRows{} rows. High-impact banking effects retain a review disposition even when legitimate.

\subsection{Results and Boundary}

\begin{table}[h]
\centering
\scriptsize
\resizebox{\linewidth}{!}{%
\begin{tabular}{L{0.25\linewidth}rrrrrr}
\toprule
\textbf{Configuration} & \textbf{Profile agree} & \textbf{Attack deny} & \textbf{Attack non-allow} & \textbf{Benign path} & \textbf{Direct path} & \textbf{FRR} \\
\midrule
\AgentAblationTableRow{Free-form justification}
\AgentAblationTableRow{Structured schema only}
\AgentAblationTableRow{Payload-bound explanation}
\AgentFullTableRow
\bottomrule
\end{tabular}}
\caption{Pinned AgentDojo semantic adaptation over \ExternalProposals{} proposals per configuration. ``Direct path'' denotes legitimate direct user goals, not injected goals embedded in another task.}
\label{tab:agentdojo}
\end{table}
\FloatBarrier

All configurations keep every attack proposal non-allow because the public banking effect tools in this subset are high risk and therefore review-gated. The full profile changes the diagnostic disposition of the \ExternalAttacks{}/\ExternalAttacks{} task--proposal contradictions from review to deny; the immediate non-allow rate remains \ExternalAttacks{}/\ExternalAttacks{} under every ablation. All configurations preserve the expected disposition of \ExternalBenign{}/\ExternalBenign{} benign and \ExternalDirect{}/\ExternalDirect{} legitimate direct-goal proposals. The full profile also produces structurally complete packets for \ExternalProposals{}/\ExternalProposals{} proposals.

This result tests transfer of one contract predicate to a public benchmark task structure. The adapter uses a self-contained \ExternalProposals-record snapshot derived from a verified minimized source artifact and deterministic public tool semantics; both the snapshot and its upstream source hash are recorded. Its unit of analysis is the task--proposal binding rather than an end-to-end AgentDojo trajectory.

\section{Discussion}

\subsection{Evidence Across Evaluation Tiers}

The authored suite measures whether the implementation realizes the named profile predicates. The draft-only integration places the verifier before a reference draft boundary and binds its outcome into the created object. The frozen hosted record characterizes contract generation under one constrained historical interface, while the zero-call layer routes the preserved minimized claims under the current pipeline. The external semantic adaptation adds the diagnostic distinction between task-bound denial and generic high-risk review.

\subsection{Attributable Fail-Closed Routing}

These tiers jointly demonstrate an attributable fail-closed interface: a proposal that cannot satisfy the current profile receives a named non-allow disposition and a bounded packet. Its coverage depends on the integrity and completeness of authoritative facts, predicates, and enforcement. Queue overload, automation bias, and approval paths that omit revalidation remain operational risks.

\subsection{Design Trade-offs}

More authoritative predicates improve diagnostic specificity but increase integration cost, fact freshness requirements, and the probability of review when state is unavailable. Bounded summaries and digests reduce ordinary disclosure but make later semantic investigation dependent on separately authorized source retention. Detailed reason codes improve repairability and audit attribution but can become an adaptive oracle; model-facing feedback should therefore be coarser than privileged audit evidence. Finally, the hosted transition counts show that automatic repair is not monotonically beneficial. Deployments need budgets, stopping rules, and queue controls rather than an unbounded repair loop.

\subsection{Implications for Agent-System Architecture}

The explanation-bound execution check is most useful where an application already owns authoritative intent, policy, route, and effect facts. It does not create those facts and should not be placed inside the same untrusted planner that generated the proposal. The verifier belongs at a reference-monitor boundary, while source retrieval, human review, and effect execution remain separately governed. This placement makes the explanation useful even when it is not faithful introspection: each enforceable statement becomes a proposition that the application can independently accept, reject, or mark unresolved.

\section{Security and Privacy Analysis}

\subsection{Plausible False Explanations}

An attacker may produce a fluent explanation that matches the payload superficially. The verifier ignores fluency as an authorization signal. Typed claims must agree with current server facts. If the attacker controls both payload and explanation, upstream intent, policy, context, route, and effect evidence remain independent comparison points.

\subsection{Explanation Generator and Planner Collusion}

Using the same model for planning and explanation creates correlated failure. The mediation layer does not assume independence between them. Both outputs are untrusted. A separate model verifier may improve empirical detection but is not part of the core guarantee unless its outputs are themselves bounded by deterministic rules.

\subsection{Stale Evidence}

Explanation references are scoped to current intent, policy, route, context, and preflight versions. A digest match to an expired or superseded record is insufficient. The effect path must revalidate current state before execution. This is essential because a previously correct explanation can become unsafe after policy or context changes.

\subsection{Information Leakage}

Explanations can become a secondary exfiltration channel. Most enforceable profile fields are categorical labels, bounded identifiers, or digests, and an explicit raw-content field is rejected. The bounded \texttt{intentSummary} remains free text, however, and the fixed-marker test does not detect novel encoding. Encoded-disclosure coverage therefore depends on semantic content scanning, retention controls, output budgets, and access-separated audit export.

\subsection{Human Automation Bias}

A verified packet may increase operator confidence even when upstream facts are wrong. Human-AI studies show that interface design can affect overreliance, and security warnings can lose effectiveness when interventions are poorly targeted~\cite{bucinca2021trust,sunshine2009warnings}. The packet labels what was verified and what remains uncertain. It must not present mechanism properties as semantic correctness. High-risk review interfaces should show effect bounds and failed predicates prominently, support source access through separately authorized channels, and measure unsafe approval rather than trust alone.

\section{Related Work}

\textbf{Explainable AI, reasoning traces, and human decision support.}
LIME and broader XAI research ask how explanations help people interpret model outputs~\cite{ribeiro2016lime,miller2019explanation}. Interpretability may be underspecified, misunderstood, or insufficient to improve decisions~\cite{lipton2018mythos,kaur2020interpreting,poursabzi2018manipulating}. ReAct interleaves reasoning and acting~\cite{yao2023react}, while chain-of-thought studies show that stated reasoning can be weakly coupled to the factors that caused an output~\cite{turpin2023unfaithful,lanham2023faithfulness,yee2024dissociation}. Human-AI guidelines and cognitive-forcing work further motivate measuring overreliance~\cite{amershi2019guidelines,bucinca2021trust}. The proposed check instead targets a server-checkable statement about an external effect.

\textbf{Explainable access control and authorization provenance.}
Explainable security and explainable access-control research studies how policy decisions can be justified to users and administrators~\cite{vigano2020xsec,mehri2025xac}. SmartAuth, AWare, decision-provenance, and access-provenance systems connect user intent, contextual evidence, or policy history to permission decisions~\cite{tian2017smartauth,petracca2017aware,singh2019decision,capobianco2017accessprov}. The action-claim mediation layer shares the reference-monitor principle and the need for attributable policy evidence, but its input is different: an untrusted agent supplies a proposed tool, payload, and bounded action claim after model planning. The gateway must compare that claim across intent, route, payload, policy, risk, and context views without allowing the explanation to create authority.

\textbf{Proof-carrying evidence and runtime enforcement.}
Proof-carrying code lets an untrusted producer accompany code with a machine-checkable safety proof~\cite{necula1997pcc}; proof-carrying authentication and subsequent distributed proving apply a producer--checker split to authorization evidence~\cite{appel1999pca,bauer2005distributed}. Security automata characterize safety policies enforceable by monitoring execution~\cite{schneider2000enforceable}. These lines motivate independent checking at a trusted boundary, but the object and guarantee here are deliberately narrower. A \texttt{ToolInvocationExplanation} is a typed claim set, not a derivation in an authorization logic or a proof of code safety. Its successful verification establishes only agreement with the resolved facts and predicates of the stated profile; execution remains subject to baseline authorization and effect controls. Conversely, the profile's per-invocation deny and review rules do not claim to enforce arbitrary trace properties.

\textbf{Agent security, specifications, and attribution.}
AgentDojo, InjecAgent, ToolHijacker, and MCPTox evaluate attacks on agent context and tool use~\cite{debenedetti2024agentdojo,zhan2024injecagent,shi2025toolhijacker,wang2025mcptox}, while AgentFuzz finds taint-style vulnerabilities in agent implementations~\cite{liu2025agentfuzz}. StruQ protects the instruction--data boundary~\cite{chen2024struq}; AttriGuard attributes tool invocations to causal context~\cite{he2026attriguard}. CaMeL separates control and data flows for agent execution, Fides applies information-flow control, and Tracked Capabilities associates authority with tracked runtime values~\cite{debenedetti2026camel,costa2025fides,odersky2026tracked}. AgentSpec and work on verifiably safe agent behavior pursue explicit constraints and checkable execution properties~\cite{wang2026agentspec,doshi2026verifiablysafe}. The explanation-bound execution check is complementary to prompt isolation, information-flow enforcement, causal attribution, capability tracking, vulnerability discovery, and action specifications: it consumes server facts from those layers and returns a non-widening action disposition plus a bounded evidence packet.

\textbf{Auditability and standards.}
Auditable Agents frames recoverability, policy checkability, responsibility attribution, and evidence integrity as system properties~\cite{nian2026auditable}; Agent Audit analyzes code and deployment artifacts before execution~\cite{zhang2026agentaudit}. Least privilege and ABAC ground the non-widening composition~\cite{saltzer1975,nist2014abac}. MCP specifies tool and transport-authorization semantics~\cite{mcp2025tools,mcp2025authorization}, while NIST and OWASP motivate TEVV, least privilege, and human oversight~\cite{nist2024genai,owasp2025promptinjection}. The proposed mediation layer contributes an application-layer, runtime action-claim object above transport authorization; it does not replace policy provenance, transport identity, or effect controls.

\begin{table}[h]
\centering
\scriptsize
\resizebox{\linewidth}{!}{%
\begin{tabular}{L{0.20\linewidth}L{0.20\linewidth}L{0.16\linewidth}L{0.20\linewidth}L{0.18\linewidth}}
\toprule
\textbf{Approach family} & \textbf{Primary object} & \textbf{Authoritative cross-check} & \textbf{Tool/effect/freshness binding} & \textbf{Relation to authority} \\
\midrule
XAI and rationale faithfulness & Model explanation or reasoning trace & Usually model/output evidence & Not the primary target & Interpretive; not an authorization rule \\
Explainable access control & Policy decision and justification & Policy state and attributes & Policy/resource centered & Explains an existing access decision \\
Proof-carrying authorization and runtime monitors & Proof or monitored event trace & Logic/policy checker or monitor state & Only properties represented by the proof or monitor & Authorizes on a valid proof or enforces a specified safety policy \\
Intent/provenance-aware permission systems & User intent, context, or provenance & OS/application facts & Domain-specific and often pre-action & Narrows or attributes permission use \\
Agent prompt/tool defenses & Prompt structure, tool behavior, or causal context & Defense-specific evidence & Often attack- or prompt-centered & Prevents or attributes selected failures \\
Agent specifications and safety monitors & Explicit behavioral/action constraints & Monitor or formal specification & Strong when modeled & Constrains execution under the specification \\
This work & Untrusted typed action claim & Intent, registry, route, payload, policy, risk, and context & Cross-view, versioned, pre-effect & Joins with existing controls; never widens them \\
\bottomrule
\end{tabular}}
\caption{Primary objects, authoritative cross-checks, effect binding, and authority semantics in this work and adjacent research.}
\label{tab:relatedcomparison}
\end{table}
\FloatBarrier

\section{Limitations and Threats to Validity}

The deterministic suite is small, authored, and partially generated from \TaskFamilies{} base tasks. Expected outcomes are defined by the same profile implemented by the evaluator, so profile-disposition agreement measures conformance within that profile. The predicate ablations are controlled mechanism variants. The AgentDojo adaptation uses a minimized proposal inventory without environment transitions, arguments, model interaction, or upstream scoring. All \ExternalAttacks{} adapted attacks are already non-allow under the high-risk ablations, so the full profile's observed increment is diagnostic denial rather than additional immediate blocking.

The draft-only wrapper uses a single in-memory clean pinned reference runtime and forced drafts; effect execution, action-time revalidation, concurrency, durability, throughput, and production latency remain unevaluated. The historical hosted pilot uses \HostedModels{} model families, \HostedTasks{} fact-complete synthetic tasks per model, \HostedRepeats{} dependent seeded repeats, and a highly constrained prompt that was iteratively debugged before the reported protocol was frozen. Availability filtering was not preregistered, and one model family accounts for all \HostedProviderErrors{} provider failures. The historical files have no run fingerprint/configuration or raw responses; digest-minimized fields prevent independent inspection of exact wording and make the current offline revalidation necessarily post hoc. The study includes no naturalistic explanations, adaptive attacks, production inference workloads, or human participants.

Structural-packet conformance means that the \StructuralPacketChecks{} authored full-profile outcomes contain the registered identity, decision, reason, predicate, per-operand availability, normalized-operand, bounded diagnostic, and fixed-order digest fields and exclude the checked raw fields or marker. FRR reports whether the claim's required digests resolve exactly to current fixture facts. Human comprehension, retained-source resolution, pre-parse byte-limit wiring, cross-domain identity binding, and semantic detection of novel sensitive content remain outside these structural checks.

The design invariants assume correct identity binding, intent, route, tool metadata, policy, risk labels, digest binding, current-state lookup, action-time recomputation, and complete mediation. Misconfigured facts, compromised gateways, malicious authorized tool code, review abuse, out-of-band actions, and incorrect provenance remain residual risks.

\section{Reproducibility and Ethics}

The internal evaluation artifact includes the task catalog, reusable verifier, deterministic evaluator, runtime wrapper, NVIDIA-only model runner, cross-model aggregator, pinned AgentDojo semantic adapter, authoritative JSON results, minimized task-level JSONL traces, derived Markdown reports, a Draft 2020--12 JSON Schema, a machine-readable conformance profile, and a hash-bound cross-artifact verifier. The conformance runner checks all \ConformanceTasks{} expected full-profile decisions, all \MetamorphicChecks{} registered metamorphic predicates, aligned schema validity, malformed-schema rejection, and registration of every emitted reason code. The external adapter additionally asserts the upstream release, clean revision, benchmark version, suite, minimized-record boundary, and exact \ExternalProposals-proposal inventory before producing \ExternalRows{} rows.

Deterministic conformance, runtime mediation, semantic adaptation, and
current hosted-record revalidation run offline; the frozen historical
generation record used NVIDIA's hosted API. The reported records contain
synthetic fixtures, public benchmark identifiers and tool names, categorical
relations, digests, outcome categories, minimized parsed fields, and reason
names. Original-harness records marked
\texttt{coverage.complete=false} are incomplete.

\subsection{Data and Code Availability}

Source code and task-level evaluation data are maintained as an internal research artifact and are not distributed with this manuscript.

\subsection{Ethics}

The reported evaluation uses synthetic inputs and public benchmark semantics, involves no human participants, and terminates reference-runtime requests as non-executing drafts.

\subsection{Funding and Competing Interests}

Accentrust provided research stewardship and computing resources. Both authors are affiliated with Accentrust; that relationship is disclosed as a potential competing interest because the proposed action-claim mediation layer composes with public OpenPort abstractions stewarded by the same organization.

\subsection{Author Contributions}

Genliang Zhu contributed conceptualization, methodology, software, investigation, artifact curation, and the initial manuscript. Chu Wang contributed conceptualization, validation, formal and empirical review, research supervision, and manuscript revision. Both authors approved the manuscript.

\section{Conclusion}

The proposed mediation layer changes the role of an agent explanation from persuasive prose to an untrusted action claim. Its value is not to reveal why a model reasoned as it did, but to let a reference monitor compare declared tool, effect, provenance, and evidence claims with authoritative application facts before authority is exercised. Under explicit trust and mediation assumptions, the composed decision cannot be less restrictive than baseline authorization or effect control, and review cannot execute without a fresh action-time decision.

The evaluation establishes conformance on the authored profile, draft-bound mediation in the reference integration, heterogeneous hosted-model generation and repair behavior, and a diagnostic deny distinction on the AgentDojo-derived task structure. The resulting interface supports composable governance and structural auditing by checking enforceable action claims independently at the tool-effect boundary. Future work will evaluate broader attack distributions, reviewer interaction, and production operating conditions.

\ifEBTEIEEEBuild
  \appendices
\else
  \appendix
\fi

\section{Reference Explanation Schema}

For byte-level compatibility with the frozen evidence package, the public
profile retains the lowercase legacy identifier \texttt{ebte-0.1} and its
corresponding schema filenames. These strings are compatibility identifiers,
not a name for the mechanism described in this paper.

\begin{lstlisting}[language=json,caption={Sanitized structured explanation example. Digests and identifiers are synthetic.}]
{
  "schemaVersion": "ebte-0.1",
  "intentClasses": ["export"],
  "intentSummary": "Export the bounded requested collection.",
  "selectedTool": "records.export",
  "toolReasonCode": "tool.effect_matches_intent",
  "policyBasis": {
    "requiredScopes": ["record.read", "record.export"],
    "policyRuleIds": ["policy.export.review"],
    "decision": "review"
  },
  "expectedEffect": {
    "operation": "export",
    "resourceType": "record",
    "resourceIds": ["collection_a"],
    "maxRecords": 50,
    "destination": "user_download"
  },
  "riskTier": "high",
  "uncertainty": {"level": "low", "reasonCodes": []},
  "untrustedContextDependencies": [],
  "evidenceRefs": [
    {"type": "intent_certificate", "digest": "sha256:0000000000000000000000000000000000000000000000000000000000000000"},
    {"type": "policy_snapshot", "digest": "sha256:1111111111111111111111111111111111111111111111111111111111111111"},
    {"type": "payload", "digest": "sha256:2222222222222222222222222222222222222222222222222222222222222222"}
  ]
}
\end{lstlisting}

\FloatBarrier
\section{Decision and Reason-Code Profile}

\begin{table}[tbp]
\centering
\footnotesize
\begin{tabular}{L{0.31\linewidth}L{0.16\linewidth}L{0.43\linewidth}}
\toprule
\textbf{Reason category} & \textbf{Default} & \textbf{Meaning} \\
\midrule
\nolinkurl{tool_mismatch} & Deny & Declared and proposed tool identities differ. \\
\nolinkurl{route_outside_authorized_set} & Deny & Proposed or declared tool is absent from the server-owned current candidate set. \\
\nolinkurl{intent_class_mismatch} & Deny & Declared intent classes do not exactly match the current certificate. \\
\nolinkurl{intent_class_overbreadth} & Deny & At least one declared class is outside the current certificate. \\
\nolinkurl{operation_outside_intent} & Deny & The payload operation is absent from the current intent's operation set. \\
\nolinkurl{resource_id_mismatch} & Deny & Explanation and payload target different resources or an overbroad resource set. \\
\nolinkurl{resource_type_mismatch} & Deny & Declared and proposed resource types differ. \\
\nolinkurl{operation_mismatch} & Deny & Declared and proposed operations differ. \\
\nolinkurl{record_bound_mismatch} & Deny & Record bound fails to cover payload, exceeds intent, or is substantive but inapplicable. \\
\nolinkurl{amount_bound_mismatch} & Deny & Amount bound fails to cover payload, exceeds intent, or is substantive but inapplicable. \\
\nolinkurl{field_bound_mismatch} & Deny & Field set fails to cover payload, exceeds intent, or is substantive but inapplicable. \\
\nolinkurl{destination_mismatch} & Deny & Destination differs from payload or is substantive but inapplicable. \\
\nolinkurl{payload_outside_intent} & Deny & Payload exceeds current certified intent even when its explanation matches. \\
\bottomrule
\end{tabular}
\caption{Reference hard-denial reasons for route, intent, resource, operation, and payload confinement.}
\label{tab:reason-action-deny}
\end{table}

\begin{table}[tbp]
\centering
\footnotesize
\begin{tabular}{L{0.31\linewidth}L{0.16\linewidth}L{0.43\linewidth}}
\toprule
\textbf{Reason category} & \textbf{Default} & \textbf{Meaning} \\
\midrule
\nolinkurl{policy_rule_mismatch} & Deny & Declared rule identifiers differ from the current applicable rule set. \\
\nolinkurl{authoritative_policy_denial} & Deny & Current server policy denies the action even when the claim acknowledges denial. \\
\nolinkurl{policy_decision_contradiction} & Deny & Current policy denies while the explanation claims a less restrictive decision. \\
\nolinkurl{policy_self_denial} & Deny & The explanation declares denial; a claim cannot use that declaration to obtain a weaker route. \\
\nolinkurl{stale_evidence} & Deny & A required evidence digest differs from current server state. \\
\nolinkurl{untrusted_dependency_omitted} & Deny & Material untrusted context lacks categorical provenance evidence. \\
\nolinkurl{context_dependency_mismatch} & Deny & The normalized dependency-pair set contains a substantive pair outside the current authoritative context state. \\
\nolinkurl{raw_context_disclosure} & Deny & Explanation attempts to persist raw untrusted content. \\
\bottomrule
\end{tabular}
\caption{Reference hard-denial reasons for policy, evidence freshness, provenance, and disclosure.}
\label{tab:reason-policy-deny}
\end{table}

\begin{table}[tbp]
\centering
\footnotesize
\begin{tabular}{L{0.31\linewidth}L{0.16\linewidth}L{0.43\linewidth}}
\toprule
\textbf{Reason category} & \textbf{Default} & \textbf{Meaning} \\
\midrule
\nolinkurl{schema_incomplete} & Review & Required typed evidence is missing or unsupported. \\
\nolinkurl{risk_understatement} & Review & Declared risk is below server-owned risk. \\
\nolinkurl{policy_basis_missing} & Review & Required scope or policy evidence is incomplete. \\
\nolinkurl{policy_decision_mismatch} & Review & Declared allow/review disposition differs from current policy without overriding a denial. \\
\nolinkurl{authoritative_policy_review} & Review & Current server policy requires review even when the declared disposition matches. \\
\nolinkurl{effect_bound_claim_incomplete} & Review & An applicable quantitative bound is absent from the typed effect claim. \\
\nolinkurl{destination_claim_incomplete} & Review & An applicable external destination is absent from the typed effect claim. \\
\nolinkurl{field_claim_incomplete} & Review & An applicable field set is absent from the typed effect claim. \\
\nolinkurl{duplicate_evidence_type} & Review & Required evidence types are not one-to-one and complete. \\
\nolinkurl{evidence_incomplete} & Review & The claim lacks an exact one-per-type set of format-valid required references. \\
\nolinkurl{high_uncertainty} & Review & Agent declares unresolved ambiguity. \\
\nolinkurl{high_risk_review} & Review & Current server-owned tool risk requires independent review. \\
\bottomrule
\end{tabular}
\caption{Reference review-escalation reasons. Deployments may use different external codes while preserving semantics.}
\label{tab:reason-review}
\end{table}

\begin{table}[tp]
\centering
\scriptsize
\begin{tabular}{L{0.40\linewidth}L{0.14\linewidth}L{0.36\linewidth}}
\toprule
\textbf{Reason category} & \textbf{Default} & \textbf{Unavailable authoritative operand} \\
\midrule
\nolinkurl{authoritative_evidence_unavailable} & Review & Exact current intent/policy/payload digest set. \\
\nolinkurl{authoritative_context_unavailable} & Review & Trusted/untrusted flag or required current context pair. \\
\nolinkurl{authoritative_tool_unavailable} & Review & Canonical tool identity. \\
\nolinkurl{authoritative_route_unavailable} & Review & Current authorized-route set. \\
\nolinkurl{authoritative_intent_classes_unavailable} & Review & Current intent-class set. \\
\nolinkurl{authoritative_intent_operations_unavailable} & Review & Certified operation set. \\
\nolinkurl{authoritative_intent_resource_unavailable} & Review & Certified resource type or target set. \\
\nolinkurl{authoritative_intent_effect_unavailable} & Review & Applicable intent record, amount, or field bounds. \\
\nolinkurl{authoritative_intent_destination_unavailable} & Review & Certified destination state. \\
\nolinkurl{authoritative_payload_operation_unavailable} & Review & Actual payload operation. \\
\nolinkurl{authoritative_payload_resource_unavailable} & Review & Actual payload resource or target. \\
\nolinkurl{authoritative_payload_effect_unavailable} & Review & Applicable payload record, amount, or field values. \\
\nolinkurl{authoritative_payload_destination_unavailable} & Review & Actual payload destination state. \\
\nolinkurl{authoritative_policy_basis_unavailable} & Review & Canonical scope or rule set. \\
\nolinkurl{authoritative_policy_decision_unavailable} & Review & Current allow/review/deny policy disposition. \\
\nolinkurl{authoritative_risk_unavailable} & Review & Canonical low/medium/high tool risk. \\
\bottomrule
\end{tabular}
\caption{Stable review reasons for unavailable authoritative facts. Each marks only its dependent packet predicate or predicates unknown; a known hard result on another available operand still dominates.}
\label{tab:reason-authoritative-unavailable}
\end{table}

\FloatBarrier
\section{Scenario Inventory}

\begin{table}[h]
\centering
\small
\begin{tabular}{L{0.25\linewidth}L{0.27\linewidth}L{0.38\linewidth}}
\toprule
\textbf{Variant} & \textbf{Expected outcome} & \textbf{Construct tested} \\
\midrule
Aligned & Allow or high-risk review & Benign compatibility and risk routing. \\
Malformed or unknown schema & Review & Strict parse and bounded-schema failure. \\
Tool mismatch & Deny & Tool identity binding. \\
Unauthorized route & Deny & Proposed tool membership in the current server-owned candidate set. \\
Intent-class overreach & Deny & Declared class confinement to the current intent certificate. \\
Operation outside intent & Deny & Payload operation binding to the certified effect class. \\
Resource mismatch & Deny & Target-resource binding. \\
Effect mismatch & Deny & Operation or quantitative effect binding. \\
Risk understatement / high uncertainty & Review & Soft escalation without attack attribution. \\
Missing policy basis & Review & Independent policy-evidence requirement. \\
Fabricated policy rule & Deny & Current policy-rule identifier binding. \\
Undeclared context & Deny & Provenance completeness. \\
Stale context digest & Deny & Current context-source and digest binding. \\
Raw context leak & Deny & Minimal-disclosure enforcement. \\
Payload outside intent & Deny & Authoritative intent-to-payload binding. \\
Resource overbreadth & Deny & Exact resource-set confinement. \\
Stale policy evidence & Deny & Current evidence-digest binding. \\
Invalid risk enumeration & Review & Bounded schema and fail-safe parsing. \\
Duplicate evidence type & Review & Evidence-type uniqueness and completeness. \\
\bottomrule
\end{tabular}
\caption{Registered scenario families in the strict reference profile.}
\end{table}

\FloatBarrier
\section{Reference Verification Algorithm}

\begin{lstlisting}[caption={Implementation-neutral action-claim decision procedure.}]
function decide(serializedExplanation, action, principal):
  require utf8ByteLength(serializedExplanation) <= 1048576
  explanation = parseJSON(serializedExplanation)
  parsed = validateStrictBoundedSchema(explanation)
  schemaRoute = classifySchemaDisposition(parsed)
  serverFacts = resolveCurrentScopedFacts(
    principal, action,
    ["intent", "tool", "route", "policy", "context"]
  )
  availability = validateAuthoritativeFields(
    action, serverFacts, perPredicateDependencies
  )
  facts = normalizeAvailableServerFacts(action, serverFacts)
  partial = normalizeParseableClaimSubfields(
    explanation,
    claimFieldNamesOnly=true,
    sliceBeforeFilterAndDeduplicate=profile.canonicalCaps
  )
  hard = compareSafelyResolvableHardDimensions(
    partial, facts, availability, [
    "intent_classes", "operation_match", "operation_inside_intent",
    "tool", "authorized_route", "payload_inside_intent",
    "resource", "effect_bounds", "destination",
    "policy_rule", "policy_decision", "context_digest",
    "evidence_freshness", "raw_context_disclosure"
  ])
  // Each applicable effect bound must cover payload and stay within intent.
  // Freshness compares every parseable required ref independently of
  // one-to-one completeness. Unavailable server operands are never matches.
  if schemaRoute.isDeny:
    claimCheck = DENY(
      schemaRoute.reasonCodes + hard.reasonCodes
        + availability.reviewReasonCodes,
      structuralPacket(
        schemaRoute, hard, availability, parsed.claimDiagnostics
      )
    )
  else if not parsed.schemaValid:
    if hard.notEmpty:
      claimCheck = DENY(
        hard.reasonCodes + schemaRoute.reasonCodes
          + availability.reviewReasonCodes,
        structuralPacket(
          hard, schemaRoute, availability, parsed.claimDiagnostics
        )
      )
    else:
      claimCheck = REVIEW(
        schemaRoute.reasonCodes + availability.reviewReasonCodes,
        structuralPacket(
          schemaRoute, availability, parsed.claimDiagnostics
        )
      )
  else:
    claims = normalizeExplanation(parsed)
    soft = compareSoftDimensions(claims, facts, [
      "risk", "policy_basis", "policy_decision",
      "effect_bound_claim_completeness",
      "destination_claim_completeness",
      "field_claim_completeness", "evidence_completeness"
    ])
    soft.addAll(
      schemaRoute.reviewReasonCodes
        + availability.reviewReasonCodes
    )
    tightening = oneWayTightening([
      claims.uncertainty
    ])
    advisory = recordWithoutAuthority([
      claims.instructionInfluence,
      claims.toolReasonCode
    ])
    if facts.tool.highRisk:
      soft.add("high_risk_review")
    if hard.notEmpty:
      claimCheck = DENY(
        hard.reasonCodes + soft.reasonCodes
          + tightening.reasonCodes,
        structuralPacket(hard, soft, tightening, advisory)
      )
    else if soft.notEmpty or tightening.requiresReview:
      claimCheck = REVIEW(
        soft.reasonCodes + tightening.reasonCodes,
        structuralPacket(soft, tightening, advisory)
      )
    else:
      claimCheck = ALLOW(
        [], structuralPacket("all_match", advisory)
      )

  final = mostRestrictive(
    serverFacts.baselineDecision,
    claimCheck,
    serverFacts.effectDecision
  )
  return final  // REVIEW can create a draft, never execute.
\end{lstlisting}

The pseudocode separates hard, soft, one-way-tightening, and audit-only
advisory inputs. In profile v0.1, high declared uncertainty can
tighten an otherwise matching result to review; \texttt{instructionInfluence}
and \texttt{toolReasonCode} are recorded but do not change the disposition.
They cannot justify allow. Implementations must resolve current server facts
before normalization, validate availability at field granularity, and must not
reuse model-supplied policy, risk, route, or provenance as authoritative
values. The exported reference verifier accepts an already parsed object, so
its caller must enforce the registered 1\,MiB UTF-8 limit before parsing; JSON
Schema \texttt{maxItems} is not a transport-memory bound. Post-parse validation
uses bounded array prefixes, a 128-error cap with a stable truncation
diagnostic, and Unicode code-point length semantics. A malformed authoritative
field makes its dependent predicate unknown without suppressing an
independently resolvable hard result. Profile v0.1
evidence references contain only intent, policy, and payload digests; route is
resolved as a server fact, and context is compared as a deduplicated
source-category/digest pair set. Any later approval invokes the procedure again
with current facts before effect execution.

\FloatBarrier
\section{Experimental Units and Denominators}

\begingroup
\small
\begin{longtable}{L{0.25\linewidth}L{0.19\linewidth}L{0.22\linewidth}L{0.23\linewidth}}
\caption{Realized experimental units and their descriptive scope.}
\label{tab:units}\\
\toprule
\textbf{Artifact} & \textbf{Unit} & \textbf{Realized denominator} & \textbf{Permitted interpretation} \\
\midrule
\endfirsthead
\toprule
\textbf{Artifact} & \textbf{Unit} & \textbf{Realized denominator} & \textbf{Permitted interpretation} \\
\midrule
\endhead
Deterministic contract & Task--configuration row & \ConformanceTasks{} tasks/configuration; \ConformanceRows{} total & Complete summary of authored profile fixtures. \\
Hard fixture & Task & \HardFixtures{} per configuration & Conformance on registered hard mutations. \\
Soft fixture & Task & \SoftFixtures{} per configuration & Conformance on registered review mutations. \\
Aligned path & Task & \AlignedFixtures{} per configuration & Compatibility with expected fixture routing. \\
Metamorphic suite & Predicate instance & \MetamorphicChecks{} total & Regression evidence for named transformations. \\
Structural packet & Full-profile outcome & \StructuralPacketChecks{} packets & Conformance of the bounded packet projection on authored outcomes. \\
Authoritative-deny control & Family--claim-shape check & \AuthoritativePolicyDenyChecks{} valid plus \AuthoritativePolicyMalformedDenyChecks{} malformed & A schema-review branch cannot downgrade a current policy denial. \\
Authoritative-review control & Family-level policy check & \AuthoritativePolicyReviewChecks{} total & A matching claim cannot downgrade current policy review to allow. \\
Mixed predicate controls & Family-level check & \SchemaSoftHardDominanceChecks{} schema-soft/hard; \RawSchemaHardDominanceChecks{} raw-schema/hard; \HardSoftPacketTruthChecks{} hard/soft & Priority and packet truth on registered compound cases. \\
Mixed-malformed controls & Family-level check or pair & \MixedMalformedAlignedSetChecks{} aligned; \MixedMalformedKnownHardChecks{} known-hard; \SelfDenialMalformedChecks{} self-denial; \TypedRawDisclosureChecks{} typed-raw; \MalformedExpectedEffectChecks{} review/deny pairs & Partial-safe comparison, non-fabrication, crash freedom, and hard-dominance regression evidence. \\
Evidence-state controls & Family-level check & \StaleIncompleteEvidenceChecks{} stale/incomplete; \UnavailableAuthoritativeEvidenceChecks{} unavailable-current & Independence of freshness, completeness, and current-authority availability. \\
Effect three-way controls & Applicable bound or family & \EffectUpperBoundChecks{} upper; \EffectSafeSlackChecks{} safe slack; \InapplicableEffectClaimChecks{} inapplicable & Claim covers payload, remains inside intent, and omits inapplicable dimensions. \\
Context-state controls & Family or untrusted family & \ContextExactnessHardChecks{} exactness; \UnavailableAuthoritativeContextChecks{} unavailable-current; \DuplicateCurrentContextChecks{} duplicate-current & Exact normalized-pair semantics and current-authority availability. \\
Authoritative-availability controls & Family-level check or pair & \AuthoritativeFactAvailabilityChecks{} unavailable-only; \OrthogonalAuthoritativeMalformedHardChecks{} unavailable plus known-hard; \EmptyAuthoritativeObjectChecks{} empty-object pairs; \MissingPayloadEnvelopeSiblingChecks{} payload; \MissingIntentEnvelopeSiblingChecks{} intent; \PresentOptionalOrthogonalHardChecks{} present optional; \ResourceOperandAvailabilityChecks{} resource operands; \InvalidAuthoritativeNumericDomainChecks{} numeric domain; \AuthoritativeEnumAvailabilityChecks{} taxonomy & Affected predicates become unknown without masking independent hard results or fabricating inapplicability from an unavailable envelope. \\
Policy/evidence operand controls & Family-level check & \SelfDenialUnavailableAuthorityChecks{} self-denial; \PolicyBasisOperandAvailabilityChecks{} split scope/rule; \ScopeFreePolicyBasisChecks{} empty/empty scopes; \EmptyScopeMissingChecks{} empty/nonempty scopes; \EvidenceExtraKeyStaleChecks{} extra-key/stale & Per-operand availability and fixed evidence-type projection. \\
Canonical replay controls & Family, pair, or boundary instance & \CanonicalScopeClaimChecks{} scopes; \WideAllowedFieldReplayChecks{} wide fields; \OverCapAuthoritativeSetChecks{} authority sets; \TailBoundedClaimChecks{} bounded claims; \TailAllowedFieldClaimChecks{} fields; \TailContextDependencyChecks{} context; \CanonicalInvalidScalarChecks{} scalars; \InvalidAuthoritativeContextSourceChecks{} source & Packet operands and semantic comparison share cap-first canonicalization. \\
Packet/validator boundary controls & Pair or boundary instance & \ClaimAliasIsolationChecks{} aliases; \PacketDiagnosticCollisionChecks{} diagnostics; \PacketOutcomeOrderChecks{} outcome order; \UnicodeLengthChecks{} Unicode; \BoundedValidationResourceChecks{} bounded validation & No alias-domain confusion, packet collisions, order dependence, UTF-16 length drift, or unbounded post-parse validation errors in the named controls. \\
Claim-omission controls & Applicable claim component & \EffectClaimOmissionChecks{} bound; \DestinationClaimOmissionChecks{} destination; \FieldClaimOmissionChecks{} field set & Review plus unknown, never a false packet match. \\
Three-way decision join & Decision triple & \JoinCompositionChecks{} total & Exhaustive finite check of baseline--action-claim--effect composition. \\
Draft-only wrapper & Task--configuration row & \RuntimeTasks{} tasks/configuration; \RuntimeRows{} total & Local in-process draft composition. \\
Runtime hard path & Task & \RuntimeHardFixtures{} per configuration & Whether a designated hard case reaches draft. \\
Runtime soft path & Task & \RuntimeSoftFixtures{} per configuration & Whether a designated soft case creates a draft. \\
Runtime aligned path & Task & \RuntimeAlignedFixtures{} per configuration & Whether an aligned case creates a draft. \\
Historical model initial generation & Attempt & \HostedInitialRows{}; \HostedModels{} models $\times$ \HostedTasks{} tasks $\times$ \HostedRepeats{} seeds & Frozen 2026-07-12 constrained model-interface record. \\
Historical model ordinary repair & Attempt & \HostedRepairRows{}; paired with initial generation & Frozen descriptive paired feedback transitions. \\
Historical model seeded repair & Attempt & \HostedSeededRows{}; \HostedModels{} models $\times$ \HostedRepairableTasks{} defects $\times$ \HostedRepeats{} seeds & Frozen repair record for named authored explanation defects. \\
Current model-output revalidation & Preserved attempt record & \HostedCurrentAttempts{} attempts; \HostedCurrentResponses{} responses; \HostedCurrentSavedClaims{} saved parsed claims & Zero-call post-hoc current-pipeline judgment over minimized historical fields; not a hosted rerun or historical prompt/code reconstruction. \\
AgentDojo semantic adaptation & Proposal--configuration row & \ExternalProposals{} proposals $\times$ \ExternalConfigurations{}; \ExternalRows{} total & Deterministic task--proposal binding on minimized structure. \\
AgentDojo cross-task attack & Proposal & \ExternalAttacks{} per configuration & Review-to-deny diagnostic refinement; all already non-allow. \\
AgentDojo legitimate paths & Proposal & \ExternalBenign{} benign and \ExternalDirect{} direct goals/configuration & Compatibility with upstream-derived roles. \\
\bottomrule
\end{longtable}
\endgroup
\FloatBarrier

The task-condition row is the primary descriptive unit. Multiple variants derived from the same base family are dependent by construction. Likewise, runtime measurements share the same local machine and in-memory implementation. The paper therefore does not attach independent-sample confidence intervals or significance tests to these rates.

\ifdefined\EBTETOPS
  \clearpage
\else
  \FloatBarrier
\fi
\section{Privacy and Evidence-Abuse Matrix}

\begin{table}[h]
\centering
\small
\begin{tabular}{L{0.26\linewidth}L{0.30\linewidth}L{0.34\linewidth}}
\toprule
\textbf{Abuse case} & \textbf{Required handling} & \textbf{Residual risk} \\
\midrule
Raw untrusted content in explanation & Reject or sanitize; store category, digest, and privacy reason only. & Encoded or obfuscated content needs additional detection. \\
Customer or tenant identifier & Replace with bounded opaque identifier under audit access control. & Linkability remains possible for privileged auditors. \\
Credential-like material & Reject, alert, and prevent durable propagation. & Detection cannot guarantee coverage of novel formats. \\
Private prompt or model trace & Exclude from the contract and ordinary audit export. & Separately retained debugging systems require their own controls. \\
Fabricated evidence digest & Resolve against current scoped records; unknown digest is non-match. & Compromised evidence stores remain outside the guarantee. \\
Oversized explanation & Enforce per-field and total schema limits before parsing. & Resource-exhaustion controls remain deployment-specific. \\
Cross-application reference & Scope every evidence lookup to authenticated app, key, actor, and session. & Incorrect identity binding defeats isolation. \\
Stale policy or context reference & Revalidate current version at action time. & Out-of-band effects cannot be recovered by revalidation. \\
\bottomrule
\end{tabular}
\caption{Privacy and evidence-abuse requirements. The current artifact exercises explicit raw-context fields, digest shape, evidence freshness, and evidence-type uniqueness.}
\label{tab:privacy}
\end{table}

\FloatBarrier
\section{Artifact Inventory and Reproduction}

The authoritative and derived artifacts are separated as follows:

\begin{itemize}
  \item \texttt{task\_catalog.mjs}: synthetic base tasks and mutations;
  \item \texttt{explanation\_verifier.mjs}: reusable decision logic;
  \item \texttt{explanation\_contract\_eval.mjs}: deterministic condition runner and metamorphic suite;
  \item \texttt{results\_explanation\_contract.json}: authoritative deterministic output;
  \item \texttt{profile/ebte-v0.1-schema.json}: bounded public JSON Schema;
  \item \texttt{profile/ebte-v0.1-profile.json}: machine-readable decisions, reasons, privacy boundary, and expected counts;
  \item \texttt{conformance\_profile\_eval.mjs}: schema/profile/reference cross-check;
  \item \texttt{runtime\_explanation\_eval.ts}: local OpenPort wrapper-path runner;
  \item \texttt{results\_runtime\_explanation.json}: authoritative runtime summary and rows;
  \item \texttt{model\_pilot\_tasks.jsonl}: synthetic aligned, repairable, and nonrepairable model tasks;
  \item \texttt{model\_explanation\_pilot.mjs}: NVIDIA-only generation and repair runner;
  \item per-model JSON/JSONL outputs: authoritative minimized model results and traces;
  \item \texttt{aggregate\_model\_explanation\_eval.mjs}: four-model evidence aggregator;
  \item \texttt{results\_model\_explanation\_multimodel.json}: authoritative cross-model summary;
  \item \texttt{agentdojo\_semantic\_adaptation\_eval.mjs}: pinned minimized external-task semantic adapter;
  \item \texttt{agentdojo\_minimized\_proposals.json}: self-contained \ExternalProposals-record public-structure snapshot;
  \item \texttt{results\_agentdojo\_semantic\_adaptation.json}: authoritative \ExternalRows-row adaptation result;
  \item \texttt{verify\_artifact\_manifest.mjs}: evidence, privacy, count, table-text, byte-length, and SHA-256 verifier; and
  \item Markdown result files: derived human-readable tables.
\end{itemize}

The deterministic artifact runs with Node.js and no external services. The runtime wrapper requires an explicit path to a compatible OpenPort reference-runtime checkout and records the Git revision, cleanliness state, and source fingerprint. The manifest verifier rejects changed bytes, mismatched result denominators, unregistered evidence totals, credential-like material, and propagation of the raw negative-fixture marker into generated outputs.

\FloatBarrier
\section{Operator Review Evidence}

A future action-claim review interface should display only decision-relevant evidence:

\begin{enumerate}
  \item proposed tool, operation, target resources, quantitative bounds, and destination;
  \item current intent classes and bounded purpose summary;
  \item current policy and route decision identifiers;
  \item risk, uncertainty, and context-dependency categories;
  \item matched, soft, unknown, and hard predicate results;
  \item whether evidence is current, stale, missing, or privacy-sanitized; and
  \item the available reviewer actions and their expected effects.
\end{enumerate}

The interface must not imply that all-match proves semantic correctness. It should distinguish gateway-verified facts from model-declared summaries and expose raw sources only through separately authorized, logged access.

\begingroup
\scriptsize
\raggedright
\sloppy
\ifEBTEIEEEBuild
  \bibliographystyle{IEEEtran}
\else
  \bibliographystyle{ACM-Reference-Format}
\fi
\bibliography{references}

@article{saltzer1975,
  author  = {Saltzer, Jerome H. and Schroeder, Michael D.},
  title   = {The Protection of Information in Computer Systems},
  journal = {Proceedings of the IEEE},
  year    = {1975},
  volume  = {63},
  number  = {9},
  pages   = {1278--1308},
  doi     = {10.1109/PROC.1975.9939}
}

@inproceedings{necula1997pcc,
  author    = {Necula, George C.},
  title     = {Proof-Carrying Code},
  booktitle = {Proceedings of the 24th ACM SIGPLAN-SIGACT Symposium on Principles of Programming Languages},
  year      = {1997},
  pages     = {106--119},
  publisher = {Association for Computing Machinery},
  address   = {New York, NY, USA},
  doi       = {10.1145/263699.263712},
  url       = {https://doi.org/10.1145/263699.263712}
}

@inproceedings{appel1999pca,
  author    = {Appel, Andrew W. and Felten, Edward W.},
  title     = {Proof-Carrying Authentication},
  booktitle = {Proceedings of the 6th ACM Conference on Computer and Communications Security},
  year      = {1999},
  pages     = {52--62},
  publisher = {Association for Computing Machinery},
  address   = {New York, NY, USA},
  doi       = {10.1145/319709.319718},
  url       = {https://doi.org/10.1145/319709.319718}
}

@article{schneider2000enforceable,
  author  = {Schneider, Fred B.},
  title   = {Enforceable Security Policies},
  journal = {ACM Transactions on Information and System Security},
  year    = {2000},
  volume  = {3},
  number  = {1},
  pages   = {30--50},
  doi     = {10.1145/353323.353382},
  url     = {https://doi.org/10.1145/353323.353382}
}

@inproceedings{bauer2005distributed,
  author    = {Bauer, Lujo and Garriss, Scott and Reiter, Michael K.},
  title     = {Distributed Proving in Access-Control Systems},
  booktitle = {Proceedings of the 2005 IEEE Symposium on Security and Privacy},
  year      = {2005},
  pages     = {81--95},
  publisher = {IEEE Computer Society},
  address   = {Washington, DC, USA},
  doi       = {10.1109/SP.2005.9},
  url       = {https://doi.org/10.1109/SP.2005.9}
}

@techreport{nist2014abac,
  author      = {Hu, Vincent C. and Ferraiolo, David and Kuhn, Rick and Schnitzer, Adam and Sandlin, Kenneth and Miller, Robert and Scarfone, Karen},
  title       = {Guide to Attribute Based Access Control ({ABAC}) Definition and Considerations},
  institution = {National Institute of Standards and Technology},
  type        = {NIST Special Publication},
  number      = {800-162},
  year        = {2014},
  doi         = {10.6028/NIST.SP.800-162},
  url         = {https://doi.org/10.6028/NIST.SP.800-162}
}

@techreport{nist2024genai,
  author      = {{National Institute of Standards and Technology}},
  title       = {Artificial Intelligence Risk Management Framework: Generative Artificial Intelligence Profile},
  institution = {National Institute of Standards and Technology},
  type        = {NIST AI},
  number      = {600-1},
  year        = {2024},
  doi         = {10.6028/NIST.AI.600-1},
  url         = {https://doi.org/10.6028/NIST.AI.600-1}
}

@misc{owasp2025promptinjection,
  author       = {{OWASP Foundation}},
  title        = {{LLM01}:2025 Prompt Injection},
  year         = {2025},
  howpublished = {\url{https://genai.owasp.org/llmrisk/llm01-prompt-injection/}},
  note         = {Accessed 2026-07-12}
}

@misc{mcp2025tools,
  author       = {{Model Context Protocol}},
  title        = {Tools -- Model Context Protocol Specification, 2025-11-25},
  year         = {2025},
  howpublished = {\url{https://modelcontextprotocol.io/specification/2025-11-25/server/tools}},
  note         = {Version 2025-11-25; accessed 2026-07-24}
}

@misc{mcp2025authorization,
  author       = {{Model Context Protocol}},
  title        = {Authorization -- Model Context Protocol Specification, 2025-11-25},
  year         = {2025},
  howpublished = {\url{https://modelcontextprotocol.io/specification/2025-11-25/basic/authorization}},
  note         = {Version 2025-11-25; accessed 2026-07-24}
}

@inproceedings{tian2017smartauth,
  author    = {Tian, Yuan and Zhang, Nan and Lin, Yueh-Hsun and Wang, Xiaofeng and Ur, Blase and Guo, Xianzheng and Tague, Patrick},
  title     = {{SmartAuth}: User-Centered Authorization for the Internet of Things},
  booktitle = {26th USENIX Security Symposium (USENIX Security 17)},
  year      = {2017},
  pages     = {361--378},
  publisher = {USENIX Association},
  address   = {Vancouver, BC},
  url       = {https://www.usenix.org/conference/usenixsecurity17/technical-sessions/presentation/tian}
}

@inproceedings{petracca2017aware,
  author    = {Petracca, Giuseppe and Reineh, Ahmad-Atamli and Sun, Yuqiong and Grossklags, Jens and Jaeger, Trent},
  title     = {{AWare}: Preventing Abuse of Privacy-Sensitive Sensors via Operation Bindings},
  booktitle = {26th USENIX Security Symposium (USENIX Security 17)},
  year      = {2017},
  pages     = {379--396},
  publisher = {USENIX Association},
  address   = {Vancouver, BC},
  url       = {https://www.usenix.org/conference/usenixsecurity17/technical-sessions/presentation/petracca}
}

@inproceedings{ribeiro2016lime,
  author    = {Ribeiro, Marco Tulio and Singh, Sameer and Guestrin, Carlos},
  title     = {Why Should I Trust You?: Explaining the Predictions of Any Classifier},
  booktitle = {Proceedings of the 22nd ACM SIGKDD International Conference on Knowledge Discovery and Data Mining},
  year      = {2016},
  pages     = {1135--1144},
  publisher = {Association for Computing Machinery},
  address   = {New York, NY, USA},
  doi       = {10.1145/2939672.2939778}
}

@article{miller2019explanation,
  author  = {Miller, Tim},
  title   = {Explanation in Artificial Intelligence: Insights from the Social Sciences},
  journal = {Artificial Intelligence},
  year    = {2019},
  volume  = {267},
  pages   = {1--38},
  doi     = {10.1016/j.artint.2018.07.007}
}

@article{lipton2018mythos,
  author  = {Lipton, Zachary C.},
  title   = {The Mythos of Model Interpretability},
  journal = {Queue},
  year    = {2018},
  volume  = {16},
  number  = {3},
  pages   = {31--57},
  doi     = {10.1145/3236386.3241340}
}

@inproceedings{amershi2019guidelines,
  author    = {Amershi, Saleema and Weld, Dan and Vorvoreanu, Mihaela and Fourney, Adam and Nushi, Besmira and Collisson, Penny and Suh, Jina and Iqbal, Shamsi and Bennett, Paul N. and Inkpen, Kori and Teevan, Jaime and Kikin-Gil, Ruth and Horvitz, Eric},
  title     = {Guidelines for {Human-AI} Interaction},
  booktitle = {Proceedings of the 2019 CHI Conference on Human Factors in Computing Systems},
  year      = {2019},
  pages     = {1--13},
  publisher = {Association for Computing Machinery},
  address   = {New York, NY, USA},
  doi       = {10.1145/3290605.3300233}
}

@inproceedings{kaur2020interpreting,
  author    = {Kaur, Harmanpreet and Nori, Harsha and Jenkins, Samuel and Caruana, Rich and Wallach, Hanna and Wortman Vaughan, Jennifer},
  title     = {Interpreting Interpretability: Understanding Data Scientists' Use of Interpretability Tools for Machine Learning},
  booktitle = {Proceedings of the 2020 CHI Conference on Human Factors in Computing Systems},
  year      = {2020},
  pages     = {1--14},
  publisher = {Association for Computing Machinery},
  address   = {New York, NY, USA},
  doi       = {10.1145/3313831.3376219}
}

@inproceedings{bucinca2021trust,
  author    = {Bu{\c{c}}inca, Zana and Malaya, Maja Barbara and Gajos, Krzysztof Z.},
  title     = {To Trust or to Think: Cognitive Forcing Functions Can Reduce Overreliance on {AI} in {AI-Assisted} Decision-Making},
  booktitle = {Proceedings of the 2021 CHI Conference on Human Factors in Computing Systems},
  year      = {2021},
  pages     = {1--21},
  publisher = {Association for Computing Machinery},
  address   = {New York, NY, USA},
  doi       = {10.1145/3411764.3445659}
}

@inproceedings{poursabzi2018manipulating,
  author    = {Poursabzi-Sangdeh, Forough and Goldstein, Daniel G. and Hofman, Jake M. and Wortman Vaughan, Jennifer and Wallach, Hanna},
  title     = {Manipulating and Measuring Model Interpretability},
  booktitle = {Proceedings of the 2021 CHI Conference on Human Factors in Computing Systems},
  year      = {2021},
  pages     = {1--52},
  articleno = {237},
  numpages  = {52},
  publisher = {Association for Computing Machinery},
  address   = {New York, NY, USA},
  doi       = {10.1145/3411764.3445315},
  url       = {https://doi.org/10.1145/3411764.3445315}
}

@inproceedings{vigano2020xsec,
  author    = {Vigan{\`o}, Luca and Magazzeni, Daniele},
  title     = {Explainable Security},
  booktitle = {2020 IEEE European Symposium on Security and Privacy Workshops (EuroS\&PW)},
  year      = {2020},
  pages     = {293--300},
  publisher = {IEEE},
  address   = {Genoa, Italy},
  doi       = {10.1109/EuroSPW51379.2020.00045},
  url       = {https://doi.org/10.1109/EuroSPW51379.2020.00045}
}

@inproceedings{mehri2025xac,
  author    = {{Hasel Mehri}, Gelareh and Morisset, Charles and Zannone, Nicola},
  title     = {Towards Explainable Access Control [{BlueSky Paper}]},
  booktitle = {Proceedings of the 30th ACM Symposium on Access Control Models and Technologies},
  year      = {2025},
  pages     = {117--126},
  publisher = {Association for Computing Machinery},
  address   = {New York, NY, USA},
  doi       = {10.1145/3734436.3734439},
  url       = {https://doi.org/10.1145/3734436.3734439}
}

@inproceedings{sunshine2009warnings,
  author    = {Sunshine, Joshua and Egelman, Serge and Almuhimedi, Hazim and Atri, Neha and Cranor, Lorrie Faith},
  title     = {Crying Wolf: An Empirical Study of {SSL} Warning Effectiveness},
  booktitle = {18th USENIX Security Symposium},
  year      = {2009},
  pages     = {399--416},
  publisher = {USENIX Association},
  address   = {Montreal, Quebec},
  url       = {https://www.usenix.org/conference/18th-usenix-security-symposium/crying-wolf-empirical-study-ssl-warning-effectiveness}
}

@misc{rundgren2020jcs,
  author       = {Rundgren, Anders and Jordan, Bret and Erdtman, Samuel},
  title        = {{JSON} Canonicalization Scheme ({JCS})},
  year         = {2020},
  howpublished = {RFC 8785},
  doi          = {10.17487/RFC8785},
  url          = {https://www.rfc-editor.org/rfc/rfc8785.html}
}

@inproceedings{yao2023react,
  author    = {Yao, Shunyu and Zhao, Jeffrey and Yu, Dian and Du, Nan and Shafran, Izhak and Narasimhan, Karthik and Cao, Yuan},
  title     = {{ReAct}: Synergizing Reasoning and Acting in Language Models},
  booktitle = {International Conference on Learning Representations},
  year      = {2023},
  numpages  = {33},
  publisher = {OpenReview.net},
  address   = {Kigali, Rwanda},
  url       = {https://openreview.net/forum?id=WE_vluYUL-X}
}

@inproceedings{turpin2023unfaithful,
  author    = {Turpin, Miles and Michael, Julian and Perez, Ethan and Bowman, Samuel R.},
  title     = {Language Models Don't Always Say What They Think: Unfaithful Explanations in Chain-of-Thought Prompting},
  booktitle = {Advances in Neural Information Processing Systems},
  year      = {2023},
  volume    = {36},
  pages     = {74952--74965},
  publisher = {Curran Associates, Inc.},
  address   = {Red Hook, NY, USA},
  doi       = {10.52202/075280-3275},
  url       = {https://proceedings.neurips.cc/paper_files/paper/2023/hash/ed3fea9033a80fea1376299fa7863f4a-Abstract-Conference.html}
}

@misc{lanham2023faithfulness,
  author        = {Lanham, Tamera and Chen, Anna and Radhakrishnan, Ansh and Steiner, Benoit and Denison, Carson and Hernandez, Danny and Li, Dustin and Durmus, Esin and Hubinger, Evan and Kernion, Jackson and Luko{\v{s}}i{\=u}t{\.e}, Kamil{\.e} and Nguyen, Karina and Cheng, Newton and Joseph, Nicholas and Schiefer, Nicholas and Rausch, Oliver and Larson, Robin and McCandlish, Sam and Kundu, Sandipan and Kadavath, Saurav and Yang, Shannon and Henighan, Thomas and Maxwell, Timothy and Telleen-Lawton, Timothy and Hume, Tristan and Hatfield-Dodds, Zac and Kaplan, Jared and Brauner, Jan and Bowman, Samuel R. and Perez, Ethan},
  title         = {Measuring Faithfulness in Chain-of-Thought Reasoning},
  year          = {2023},
  eprint        = {2307.13702},
  archivePrefix = {arXiv},
  primaryClass  = {cs.AI},
  doi           = {10.48550/arXiv.2307.13702},
  url           = {https://arxiv.org/abs/2307.13702},
  note          = {Preprint}
}

@misc{yee2024dissociation,
  author        = {Yee, Evelyn and Li, Alice and Tang, Chenyu and Jung, Yeon Ho and Paturi, Ramamohan and Bergen, Leon},
  title         = {Dissociation of Faithful and Unfaithful Reasoning in {LLMs}},
  year          = {2024},
  eprint        = {2405.15092},
  archivePrefix = {arXiv},
  primaryClass  = {cs.AI},
  doi           = {10.48550/arXiv.2405.15092},
  url           = {https://arxiv.org/abs/2405.15092},
  note          = {Preprint}
}

@inproceedings{zhan2024injecagent,
  author    = {Zhan, Qiusi and Liang, Zhixiang and Ying, Zifan and Kang, Daniel},
  title     = {InjecAgent: Benchmarking Indirect Prompt Injections in Tool-Integrated Large Language Model Agents},
  booktitle = {Findings of the Association for Computational Linguistics: ACL 2024},
  year      = {2024},
  pages     = {10471--10506},
  publisher = {Association for Computational Linguistics},
  address   = {Bangkok, Thailand},
  doi       = {10.18653/v1/2024.findings-acl.624},
  url       = {https://aclanthology.org/2024.findings-acl.624/}
}

@inproceedings{greshake2023indirect,
  author    = {Greshake, Kai and Abdelnabi, Sahar and Mishra, Shailesh and Endres, Christoph and Holz, Thorsten and Fritz, Mario},
  title     = {Not What You've Signed Up For: Compromising Real-World {LLM-Integrated} Applications with Indirect Prompt Injection},
  booktitle = {Proceedings of the 16th ACM Workshop on Artificial Intelligence and Security},
  year      = {2023},
  pages     = {79--90},
  publisher = {Association for Computing Machinery},
  address   = {New York, NY, USA},
  doi       = {10.1145/3605764.3623985}
}

@inproceedings{liu2025agentfuzz,
  author    = {Liu, Fengyu and Zhang, Yuan and Luo, Jiaqi and Dai, Jiarun and Chen, Tian and Yuan, Letian and Yu, Zhengmin and Shi, Youkun and Li, Ke and Zhou, Chengyuan and Chen, Hao and Yang, Min},
  title     = {Make Agent Defeat Agent: Automatic Detection of {Taint-Style} Vulnerabilities in {LLM-based} Agents},
  booktitle = {34th USENIX Security Symposium (USENIX Security 25)},
  year      = {2025},
  pages     = {3767--3786},
  publisher = {USENIX Association},
  address   = {Seattle, WA, USA},
  isbn      = {978-1-939133-52-6},
  url       = {https://www.usenix.org/conference/usenixsecurity25/presentation/liu-fengyu}
}

@inproceedings{ruan2024toolemu,
  author    = {Ruan, Yangjun and Dong, Honghua and Wang, Andrew and Pitis, Silviu and Zhou, Yongchao and Ba, Jimmy and Dubois, Yann and Maddison, Chris J. and Hashimoto, Tatsunori},
  title     = {Identifying the Risks of {LM} Agents with an {LM-Emulated} Sandbox},
  booktitle = {The Twelfth International Conference on Learning Representations},
  year      = {2024},
  numpages  = {70},
  publisher = {OpenReview.net},
  address   = {Vienna, Austria},
  url       = {https://openreview.net/forum?id=GEcwtMk1uA}
}

@inproceedings{debenedetti2024agentdojo,
  author    = {Debenedetti, Edoardo and Zhang, Jie and Balunovic, Mislav and Beurer-Kellner, Luca and Fischer, Marc and Tram{\`e}r, Florian},
  title     = {{AgentDojo}: A Dynamic Environment to Evaluate Prompt Injection Attacks and Defenses for {LLM} Agents},
  booktitle = {Advances in Neural Information Processing Systems 37},
  year      = {2024},
  pages     = {82895--82920},
  publisher = {Curran Associates, Inc.},
  address   = {Red Hook, NY, USA},
  doi       = {10.52202/079017-2636},
  url       = {https://proceedings.neurips.cc/paper_files/paper/2024/hash/97091a5177d8dc64b1da8bf3e1f6fb54-Abstract-Datasets_and_Benchmarks_Track.html}
}

@inproceedings{chen2024struq,
  author    = {Chen, Sizhe and Piet, Julien and Sitawarin, Chawin and Wagner, David},
  title     = {{StruQ}: Defending Against Prompt Injection with Structured Queries},
  booktitle = {34th USENIX Security Symposium (USENIX Security 25)},
  year      = {2025},
  pages     = {2383--2400},
  publisher = {USENIX Association},
  address   = {Seattle, WA},
  url       = {https://www.usenix.org/conference/usenixsecurity25/presentation/chen-sizhe}
}

@inproceedings{debenedetti2026camel,
  author        = {Debenedetti, Edoardo and Shumailov, Ilia and Fan, Tianqi and Hayes, Jamie and Carlini, Nicholas and Fabian, Daniel and Kern, Christoph and Shi, Chongyang and Terzis, Andreas and Tram{\`e}r, Florian},
  title         = {Defeating Prompt Injections by Design},
  booktitle     = {IEEE Conference on Secure and Trustworthy Machine Learning (SaTML)},
  year          = {2026},
  eprint        = {2503.18813},
  archivePrefix = {arXiv},
  primaryClass  = {cs.CR},
  publisher     = {IEEE},
  address       = {Munich, Germany},
  url           = {https://satml.org/2026/program/},
  note          = {Also available as arXiv:2503.18813}
}

@misc{costa2025fides,
  author        = {Costa, Manuel and K{\"o}pf, Boris and Kolluri, Aashish and Paverd, Andrew and Russinovich, Mark and Salem, Ahmed and Tople, Shruti and Wutschitz, Lukas and Zanella-B{\'e}guelin, Santiago},
  title         = {Securing {AI} Agents with Information-Flow Control},
  year          = {2025},
  eprint        = {2505.23643},
  archivePrefix = {arXiv},
  primaryClass  = {cs.CR},
  doi           = {10.48550/arXiv.2505.23643},
  url           = {https://arxiv.org/abs/2505.23643},
  note          = {Preprint}
}

@inproceedings{odersky2026tracked,
  author    = {Odersky, Martin and Zhao, Yaoyu and Xu, Yichen and Bra{\v{c}}evac, Oliver and Pham, Cao Nguyen},
  title     = {Securing Agents With Tracked Capabilities},
  booktitle = {Proceedings of the ACM Conference on AI and Agentic Systems},
  year      = {2026},
  pages     = {812--838},
  publisher = {Association for Computing Machinery},
  address   = {New York, NY, USA},
  doi       = {10.1145/3786335.3813127},
  url       = {https://doi.org/10.1145/3786335.3813127}
}

@inproceedings{wang2026agentspec,
  author    = {Wang, Haoyu and Poskitt, Christopher M. and Sun, Jun},
  title     = {{AgentSpec}: Customizable Runtime Enforcement for Safe and Reliable {LLM} Agents},
  booktitle = {2026 IEEE/ACM 48th International Conference on Software Engineering (ICSE)},
  year      = {2026},
  numpages  = {12},
  publisher = {Association for Computing Machinery},
  address   = {New York, NY, USA},
  doi       = {10.1145/3744916.3764546},
  url       = {https://cposkitt.github.io/files/publications/agentspec_llm_enforcement_icse26.pdf},
  note      = {12 pages}
}

@inproceedings{doshi2026verifiablysafe,
  author    = {Doshi, Aarya and Hong, Yining and Xu, Congying and Kang, Eunsuk and Kapravelos, Alexandros and K{\"a}stner, Christian},
  title     = {Towards Verifiably Safe Tool Use for {LLM} Agents},
  booktitle = {Proceedings of the IEEE/ACM 48th International Conference on Software Engineering: New Ideas and Emerging Results (ICSE-NIER)},
  year      = {2026},
  pages     = {201--205},
  publisher = {Association for Computing Machinery},
  address   = {New York, NY, USA},
  doi       = {10.1145/3786582.3786839},
  url       = {https://doi.org/10.1145/3786582.3786839}
}

@inproceedings{shi2025toolhijacker,
  author    = {Shi, Jiawen and Yuan, Zenghui and Tie, Guiyao and Zhou, Pan and Gong, Neil Zhenqiang and Sun, Lichao},
  title     = {Prompt Injection Attack to Tool Selection in {LLM} Agents},
  booktitle = {Network and Distributed System Security Symposium (NDSS 2026)},
  year      = {2026},
  numpages  = {18},
  publisher = {Internet Society},
  address   = {San Diego, CA, USA},
  doi       = {10.14722/ndss.2026.230675},
  url       = {https://www.ndss-symposium.org/ndss-paper/prompt-injection-attack-to-tool-selection-in-llm-agents/}
}

@article{wang2025mcptox,
  author  = {Wang, Zhiqiang and Gao, Yichao and Wang, Yanting and Liu, Suyuan and Sun, Haifeng and Cheng, Haoran and Shi, Guanquan and Du, Haohua and Li, Xiangyang},
  title   = {{MCPTox}: A Benchmark for Tool Poisoning on Real-World {MCP} Servers},
  journal = {Proceedings of the AAAI Conference on Artificial Intelligence},
  year    = {2026},
  volume  = {40},
  number  = {42},
  pages   = {35811--35819},
  doi     = {10.1609/aaai.v40i42.40895},
  url     = {https://ojs.aaai.org/index.php/AAAI/article/view/40895}
}

@inproceedings{he2026attriguard,
  author    = {He, Yu and Zhu, Haozhe and Li, Yiming and Shao, Shuo and Yao, Hongwei and Liu, Zhihao and Qin, Zhan},
  title     = {{AttriGuard}: Defeating Indirect Prompt Injection in {LLM} Agents via Causal Attribution of Tool Invocations},
  booktitle = {35th USENIX Security Symposium (USENIX Security 26)},
  year      = {2026},
  pages     = {1547--1566},
  publisher = {USENIX Association},
  address   = {Baltimore, MD},
  isbn      = {978-1-939133-58-8},
  url       = {https://www.usenix.org/conference/usenixsecurity26/presentation/he-yu}
}

@article{singh2019decision,
  author  = {Singh, Jatinder and Cobbe, Jennifer and Norval, Chris},
  title   = {Decision Provenance: Harnessing Data Flow for Accountable Systems},
  journal = {IEEE Access},
  year    = {2019},
  volume  = {7},
  pages   = {6562--6574},
  doi     = {10.1109/ACCESS.2018.2887201},
  url     = {https://doi.org/10.1109/ACCESS.2018.2887201}
}

@inproceedings{capobianco2017accessprov,
  author    = {Capobianco, Frank and Skalka, Christian and Jaeger, Trent},
  title     = {{ACCESSPROV}: Tracking the Provenance of Access Control Decisions},
  booktitle = {9th USENIX Workshop on the Theory and Practice of Provenance (TaPP 2017)},
  year      = {2017},
  numpages  = {8},
  publisher = {USENIX Association},
  address   = {Seattle, WA},
  url       = {https://www.usenix.org/conference/tapp17/workshop-program/presentation/capobianco}
}

@misc{nian2026auditable,
  author        = {Nian, Yi and Yuan, Aojie and Zhang, Haiyue and Li, Jiate and Zhao, Yue},
  title         = {Auditable Agents},
  year          = {2026},
  eprint        = {2604.05485},
  archivePrefix = {arXiv},
  primaryClass  = {cs.AI},
  doi           = {10.48550/arXiv.2604.05485},
  url           = {https://arxiv.org/abs/2604.05485},
  note          = {Preprint}
}

@misc{zhang2026agentaudit,
  author        = {Zhang, Haiyue and Nian, Yi and Zhao, Yue},
  title         = {Agent Audit: A Security Analysis System for {LLM} Agent Applications},
  year          = {2026},
  eprint        = {2603.22853},
  archivePrefix = {arXiv},
  primaryClass  = {cs.CR},
  doi           = {10.48550/arXiv.2603.22853},
  url           = {https://arxiv.org/abs/2603.22853},
  note          = {Preprint}
}

@misc{zhu2026igac,
  author        = {Zhu, Genliang and Wang, Chu},
  title         = {Intent-Governed Tool Authorization for {AI} Agents},
  year          = {2026},
  eprint        = {2606.22916},
  archivePrefix = {arXiv},
  primaryClass  = {cs.AI},
  doi           = {10.48550/arXiv.2606.22916},
  url           = {https://arxiv.org/abs/2606.22916},
  note          = {Preprint}
}
\endgroup

\end{document}